\documentclass[11pt]{article}

\title{Efficient Calibration for Decision Making}
\author{Parikshit Gopalan\\
Apple
\and Konstantinos Stavropoulos\thanks{Work done during an internship at Apple.}\\
UT Austin\and Kunal Talwar\\
Apple
\and Pranay Tankala\samethanks\\
Harvard}
\date{}
\usepackage[utf8]{inputenc}   
\usepackage[T1]{fontenc}
\usepackage{lmodern}          
\usepackage{geometry}         
\usepackage{microtype}        

\usepackage{mathtools, amssymb, amsthm, bbm}
\usepackage{enumitem}

\usepackage[linesnumbered,ruled,vlined]{algorithm2e}
\usepackage[noend]{algpseudocode}

\usepackage{graphicx}
\usepackage{booktabs}
\usepackage{caption}
\usepackage{subcaption}

\usepackage[dvipsnames]{xcolor}

\usepackage{xcolor}

\usepackage[backref=page,bookmarks]{hyperref}
\hypersetup{
  colorlinks=true,
  linkcolor=magenta,
  citecolor=blue,
  urlcolor=blue
}

\usepackage[nameinlink,capitalize]{cleveref}
\usepackage[numbers,sort&compress]{natbib}

\newcommand*\samethanks[1][\value{footnote}]{\footnotemark[#1]}

\theoremstyle{plain}
\newtheorem{theorem}{Theorem}[section]
\newtheorem{lemma}[theorem]{Lemma}
\newtheorem{corollary}[theorem]{Corollary}
\newtheorem{proposition}[theorem]{Proposition}

\newtheorem*{claim}{Claim}

\theoremstyle{definition}
\newtheorem{definition}[theorem]{Definition}

\theoremstyle{remark}
\newtheorem{remark}[theorem]{Remark}

\numberwithin{equation}{section}

\def\A{\mathcal{A}}

\def\C{\mathcal{C}}
\def\D{\mathcal{D}}

\def\I{\mathcal{I}}
\def\J{\mathcal{J}}
\def\K{\mathcal{K}}
\def\L{\mathcal{L}}
\def\M{\mathcal{M}}

\def\W{\mathcal{W}}

\def\X{\mathcal{X}}

\newcommand*{\N}{{\mathbb{N}}}

\newcommand*{\R}{{\mathbb{R}}}

\DeclareMathOperator*{\argmin}{arg\,min}
\newcommand{\pmo}{\{+1,-1\}}

\newcommand{\intr}{\mathsf{Int}}
\let\eps\varepsilon
\let\phi\varphi

\DeclareMathOperator*{\pr}{\mathbb{P}}

\DeclareMathOperator*{\E}{\mathbb{E}}

\let\hat\widehat
\DeclareMathOperator{\poly}{poly}

\DeclareMathOperator{\sign}{sign}

\DeclareMathOperator{\proj}{\pi}

\newcommand{\cube}[1]{\{\pm 1\}^{#1}}

\newcommand{\ignore}[1]{} 
\newcommand{\eat}[1]{} 

\newcommand{\tilda}{\tilde}

\newcommand*{\vv}{\mathbf{v}}

\newcommand{\fr}[1]{\frac{1}{#1}}
\newcommand{\lt}{\left}
\newcommand{\rt}{\right}
\newcommand{\Lip}{\mathrm{Lip}}

\newcommand{\cdl}{\mathsf{CDL}}
\newcommand{\ce}{\mathsf{CE}}
\newcommand{\smce}{\mathsf{smCE}}
\newcommand{\ece}{\mathsf{ECE}}

\newcommand{\vc}{\mathsf{VCdim}}

\newcommand{\Be}{\mathsf{Ber}} 

\newcommand{\erm}{\mathsf{ERM}}

\newcommand{\accept}{\mathsf{Accept}}
\newcommand{\reject}{\mathsf{Reject}}

\newcommand{\threshold}{\mathsf{thr}}
\newcommand{\tk}{\tilde{\kappa}}

\newcommand{\calMA}{\mathrm{calMA}}

\newcommand{\propce}{\ce_{\intr}}
\DeclarePairedDelimiter{\abs}{\lvert}{\rvert}
\usepackage{bm}
\usepackage{comment}
\definecolor{navy}{rgb}{0, 0, 0.75}
\newcommand{\pranay}[1]{{\color{navy}\textbf{Pranay:} #1}}
\newcommand{\pg}[1]{{\color{red}\textbf{Parikshit:} #1}}
\newcommand{\kunal}[1]{{\color{Green}\textbf{Kunal:} #1}}
\newcommand{\kostas}[1]{{\color{purple}\textbf{Kostas:} #1}}

\eat{
\newcommand{\pranay}[1]{}
\newcommand{\pg}[1]{}
\newcommand{\kunal}[1]{}
\newcommand{\kostas}[1]{}
}

\newcommand{\AL}{\mathsf{AL}}
\newcommand{\prev}{\mathsf{prev}}
\newcommand{\zo}{\{0,1\}}
\newcommand{\ty}{\tilde{y}}
\newcommand{\Ber}{\mathsf{Ber}}

\newcommand{\smcdl}{\mathsf{smCDL}}

\newcommand{\Kall}{{\K^*}} 
\begin{document}

\maketitle
\begin{abstract}

 A decision-theoretic characterization of perfect calibration is that an agent seeking to minimize a proper loss in expectation cannot improve their outcome by post-processing a perfectly calibrated predictor. Hu and Wu (FOCS’24) use this to define an approximate calibration measure called calibration decision loss ($\mathsf{CDL}$), which measures the maximal improvement achievable by any post-processing over any proper loss. Unfortunately, $\mathsf{CDL}$ turns out to be intractable to even weakly approximate in the offline setting, given black-box access to the predictions and labels. 
 
 We suggest circumventing this by restricting attention to structured families of post-processing functions~$\mathcal{K}$. We define the calibration decision loss relative to~$\mathcal{K}$, denoted $\mathsf{CDL}_\mathcal{K}$ where we consider all proper losses but restrict post-processings to a structured family $\mathcal{K}$. We develop a comprehensive theory of when $\mathsf{CDL}_\mathcal{K}$ is information-theoretically and computationally tractable:

\begin{itemize}
\item {\bf Complexity characterization.} The sample complexity of estimating~$\mathsf{CDL}_\mathcal{K}$ is determined by the VC dimension of~$\mathsf{thr}(\mathcal{K})$, the concept class consisting of thresholds applied to any $\kappa \in \mathcal{K}$. Computationally, estimating $\mathsf{CDL}_\mathcal{K}$ reduces to agnostically learning~$\mathsf{thr}(\mathcal{K})$. This  implies that estimating $\mathsf{CDL}$ relative to $1$-Lipschitz post-processings is information-theoretically hard.

\item {\bf Quantitative characterization.} Augmenting $\mathsf{thr}(\mathcal{K})$ with indicators of intervals of the form $[0,a]$ yields a family of weight functions $\mathcal{K}'$ such that~$\mathsf{CDL}_\mathcal{K}$ is characterized, up to a quadratic factor, by the weighted calibration error restricted to~$\mathcal{K}'$. This significantly generalizes prior bounds that were for specific choices of~$\mathcal{K}$.

\item {\bf Omniprediction.}  If $\mathsf{thr}(\mathcal{K})$ is efficiently learnable there exists a {\em single} post-processing that performs competitively with the best post-processing in~$\mathcal{K}$ for {\em every} proper loss.  Classical recalibration algorithms including the Pool Adjacent Violators (PAV) algorithm and Uniform-mass binning give similar {\em omniprediction} guarantees for natural classes of post-processings with monotonic structure.
\end{itemize}

In addition to introducing new definitions and algorithmic techniques to the theory of calibration for decision making, our results give rigorous guarantees for some widely used recalibration procedures in machine learning.
\end{abstract}

\thispagestyle{empty}
\newpage

\tableofcontents
\thispagestyle{empty}
\newpage

\setcounter{page}{1}

\section{Introduction}

Consider the setting of binary classification, where we see examples $(x,y)$ drawn from a distribution $\D$ on $\X \times \zo$. A predictor $P:\X \to [0,1]$ estimates the probability that the label $y = 1$ for a given $x$. While predictions take values in the interval $[0,1]$, the labels are binary. What does it mean for a predictor to be good in such a setting? 

The notion of (perfect) calibration, which originates in the forecasting literature \cite{dawid1985calibration} requires that every predicted value $p$ in the range of $P$, we have $\E[y|P(x)=p] = p$. The Bayes optimal predictor, defined as $p^*(x) = \E[y|x]$ is calibrated, but calibration is strictly weaker than Bayes optimality. Despite this, calibration gives important guarantees for downstream decision makers who make decisions based on the predictions, where they can trust a calibrated predictor and  act {\em as though it were indeed the Bayes optimal}.

Consider an agent who uses the predictions $P(x)$ to choose an action $a \in \A$, so as to minimize their expected loss $\E[\ell(a,y)]$ for some loss function $\ell : \A \times \{0, 1\} \to \R$. Such an agent can respond according to the best-response function $\kappa_\ell:[0,1] \to \A$, where if we denote by $\Ber(p)$ the Bernoulli distribution with parameter $p$, then
\[ \kappa_\ell(p) = \argmin_{a \in \A}\E_{\ty \sim \Ber(p)}  [\ell(a,\ty)]. \]
If we wished to minimize loss for the labels $\ty \sim \Ber(p)$, where $p = P(x)$, then $P$ is Bayes optimal. So this corresponds to the agent trusting the predictor $P$ even for labels $y$, as though it were the Bayes optimal predictor. 

Calibration gives two important guarantees to any such agent, for every loss function $\ell$:

\begin{itemize}
    \item {\bf No surprises: } The expected loss suffered $\E[\ell(\kappa(p), y)]$ under the true labels equals the expected loss $\E[\ell(\kappa(p), \ty)]$ they would expect to suffer if $P$ was indeed Bayes optimal. 

    \item {\bf No regrets: } The expected loss $\E[\ell(\kappa(p), y)]$ under the true labels is indeed minimized by playing the best response $\kappa^*$ over any other function $\kappa: [0,1] \to \A$.
\end{itemize}

While our work focuses on calibration guarantees for decision making, in the the broader context, recent interest in calibration has been driven by the numerous surprising applications of calibration and its generalizations to algorithmic fairness \cite{hkrr2018}, learning  \cite{omni}, complexity theory \cite{CDV}, pseudorandomness \cite{OI}, cryptography \cite{HuVadhan} and other areas of theoretical computer science.

Perfect calibration is an idealized notion; we cannot realistically expect it from predictors in the real world for computational and information-theoretic reasons. This has motivated the formulation of approximate notions of calibration \cite{FosterV98, kakadeF08, ZhaoKSME21, BlasiokGHN23, when-does, BlasiokN24, HuWu24, OKK25, RosselliniSBRW25}.  To be useful, approximate notions of calibration should be computationally efficient, and yield some relaxed form of the guarantees above. We refer here to the efficiency of auditing for calibration error, where the goal is to estimate a calibration measure to within a prescribed additive error, from random samples of the form $(p, y)$. Following common practice in the literature (see e.g. \cite{BlasiokGHN23}), we henceforth drop $x$ from our notation and consider the joint distribution $(p, y)$. While $x$ influences how $p = P(x)$ and $y = y|x$ are jointly distributed, we do not mention it explicitly, to emphasize the fact that calibration measures (like loss functions) are typically independent of the feature space.

\subsection{Approximate Calibration From Indistinguishability.}

This is a view of calibration  as a notion of indistinguishability between the real world and a simulation that the predictor $P$ proposes. It is based on the outcome indistinguishability framework of \cite{OI} and developed in \cite{GopalanKSZ22, GH-survey}. The real world is modeled by the joint distribution $J =(p, y)$ and the simulation by $\tilde{J} = (p, \ty)$. Perfect calibration is equivalent to the two distributions being identical.

This naturally suggests relaxations that only require the distributions to be close, not identical, but also restrict the set of distinguishers to {\em fool}, in analogy with cryptography and pseudorandomness. This restriction turns out to be crucial for efficient auditing. This relaxation is captured by the definition of weight-restricted calibration \cite{GopalanKSZ22}, where we consider a family of functions $\W \subseteq \{[0,1] \mapsto [-1,1]\}$ and require that the distributions $(p,y)$ and $(p, \ty)$ be indistinguishable to all predictors of the form $\{f(p, y) = w(p)y\}_{w \in \W}$.\footnote{All bounded distinguishers $f(p,y)$ can be assumed to have the form $w(p)y$, see \cite{GH-survey}.} Formally, for a distribution $J = (p, y)$, we define the $\W$-restricted calibration error as
\[ \ce_{\mathcal{W}}(J)= \max_{w \in \W}\E[w(p)(y - \ty)] = \max_{w \in \W}\E[w(p)(y -p)] .\]
The No Surprises property, which is referred to as decision OI in the literature \cite{gopalan2023loss}, can be seen as a form of indistinguishability between these distributions. 

A sequence of works in recent years \cite{GopalanKSZ22, BlasiokGHN23, when-does, BlasiokN24, OKK25, RosselliniSBRW25} have painted a complete picture of  weight families for which $\W$-restricted calibration is meaningful and tractable; we refer the reader to the recent survey \cite{GH-survey}. We highlight a few results that are relevant to us:
\begin{enumerate}
    \item {\bf Complexity of auditing.} The tractability of estimating $\ce_\W$ is tightly captured by the learnability of the class $\W$ \cite{GopalanHR24}. 
    \item {\bf Expected calibration Error.} The expected calibration error $\ece(J)$ corresponds to $\W = \W^*$ being all bounded functions. Since $\W^*$ has unbounded VC dimension, $\ece$ cannot be audited from finitely many samples \cite{GopalanHR24}.
    \item {\bf Decision OI.} Decision OI is characterized by $\ce_\intr$ where $\intr$ is the set of indicators of intervals in $[0,1]$ \cite{OKK25}.\footnote{This notion is called Proper calibration by \cite{OKK25} and cutoff calibration by \cite{RosselliniSBRW25}, we will refer to it as Interval-restricted calibration $\ce_\intr$ in keeping with the weight-restricted calibration nomenclature.}
    \item {\bf Smooth calibration.} Considering the family of $1$-Lipschitz weight functions gives smooth calibration \cite{kakadeF08}, which is robust under small perturbations of the predictor, and captures an intuitive measure of calibration error called the distance to calibration \cite{BlasiokGHN23}.
\end{enumerate}
Since both intervals and Lipschitz functions are efficiently learnable in $1$ dimension, the latter two notions can be audited efficiently (see also \cite{hu2024testing}).  

\subsection{Approximate Calibration From No Regret.} An alternate view of approximate calibration comes from relaxing the no regret property of perfect calibration. We will restrict our attention to (bounded) proper loss functions. Informally, for a loss function to be proper, a predictor with knowledge of the true outcome distribution should not be able to further decrease its loss by dishonestly predicting some alternate distribution. Formally, these are loss functions where the action space $\A$ is the space of predictions $[0,1]$, and when $y \sim \Ber(q)$, the prediction $p$ that minimizes the expected loss $\E[\ell(p, y)]$ is $p = q$. Restricting to proper losses is without loss of generality since an arbitrary loss function $\ell$ can be converted to a proper loss $\ell'$ by composing it with the best response $\ell' = \ell \circ \kappa_\ell^* $.\footnote{In other words, a predictor that knows that an agent will best-respond to their predictions can view the agent as optimizing a proper loss.} We let $\L^*$ denote the set of all bounded proper loss functions, and $\Kall = \{[0,1] \mapsto [0,1]\}$ denote the family of all post-processings.

The starting point is the following characterization of calibration due to \cite{FosterV98}: $J = (p,y)$ is perfectly calibrated iff for every proper loss $\ell \in \L^*$ and post-processing $\kappa \in \Kall$,
\[ \E[\ell(p, y) \leq \E[\ell(\kappa(p), y)].\]
This means that in every miscalibrated predictor, this inequality is violated for some $\ell, \kappa$ pair. The recent work of Hu and Wu \cite{HuWu24}, building on \cite{kleinberg2023u}, suggests using the magnitude of this violation as a measure of miscalibration, which they call the {\em calibration decision loss}. Formally, for a family of post-processings $\K \subseteq \Kall$,\footnote{We assume $\K$ contains the identity function, which ensure the positivity of $\cdl_\K$.} we define the {\em calibration decision loss relative to $\K$} as 
\[ \cdl_{\K}(J) = \max_{\kappa \in K}\E[\ell(p, y) - \ell(\kappa(p), y)]. \]
This leads to a natural family of calibration measures parametrized by $\K$, which measures how much regret (excess loss) a decision maker could possibly suffer for a lack of calibration, relative to a baseline set of post-processings $\K$. In this work we propose studying the complexity of $\cdl_\K$ for general families $\K$ of post-processing functions. 

Motivated by the online prediction setting, \cite{HuWu24} were primarily interested in the case $\K = \Kall$, where they show a tight connection with the $\ece$:
\[\ece(J)^2 \lesssim \cdl_{\Kall}(J) \lesssim \ece(J).\]
In the online setting, they showed surprisingly that $\cdl$ admits much lower regret rates than $\ece$. However, for the auditing problem in the offline setting which is our main focus, this tight connection to $\ece$ means that $\cdl_{\Kall}$ cannot be audited from finitely many samples. One might hope to circumvent this intractability by suitably restricting the family of post-processings $\K$, in analogy to the weight-restricted calibration setting. The family of monotone post-processings $\M_+$ was considered in  \cite{RosselliniSBRW25}, who bound the calibration decision loss relative to $\M_+$ by $\ce_\intr$, where $\intr$ is the family of indicators of intervals, showing that $\cdl_{\M_+}(J) \leq 2\ce_{\intr}(J)$.\footnote{While not stated in terms of $\cdl$, Proposition 3.2 in their paper is equivalent to the above inequality.} 

In addition to being a natural question from a theoretical viewpoint, understanding $\cdl_\K$ for various classes $\K$ is relevant to how calibration errors are measured and remediated in practice \cite{guo2017calibration}. Popular methods for recalibration such as Isotonic Regression \cite{zadrozny2001obtaining, zadrozny2002transforming} and the Pool Adjacent Violators (PAV) algorithm \cite{ayer1955empirical} and Platt scaling \cite{Platt}, try to find the best post-processing from a family $\K$ (see \cite{Niculescu-MizilC05} for more details about these methods). Isotonic regression (which is the problem solved by PAV) considers monotone post-processings, whereas Platt scaling considers a parametrized subclass of monotone functions consisting of sigmoids.  Implicitly, such methods consider the predictor calibrated if $\cdl_\K$ is small. Other families of post-processings like recalibration based on Uniform-mass binning \cite{zadrozny2001obtaining,gupta2021distribution,sun2023minimum} and vector/matrix scaling have been proposed \cite{guo2017calibration}. Methods that find the best post-processing from a small class have also shown to be effective in practice (see e.g.~\cite{BertaHJB25}). However, there was a lack of theory to guide the choice of $\K$. 

Other results in the literature on decision making from calibration correspond to studying $\cdl$ for restrictions of post-processings, loss functions or both. For instance, \cite{when-does} consider Lipschitz post-processings, and show that the maximum improvement to the expected squared loss $\ell_2(p, y) = (y - p)^2$ from such post-processings is quadratically related to the smooth calibration error.   But in contrast to weight-restricted calibration, there wasn't a comprehensive understanding of when calibration decision loss relative to post-processings in $\K$ is a tractable measure. 

\paragraph{Our work:}  In this work, we seek to understand the families of post-processings $\K$ for which $\cdl_\K$ gives a tractable calibration measure that has strong guarantees for downstream decision makers. This raises several natural questions:
\begin{enumerate}
    \item {\bf Complexity characterization.} For what families of post-processings $\K$ is estimating $\cdl_{\K}$ tractable in terms of sample and computational complexity? Is there a complexity measure for $\K$ that governs tractability?
    \item {\bf Specific post-processings.} Is $\cdl$ estimation tractable for  monotone post-processings and Lipschitz post-processings?\footnote{We invite the reader to make a guess about these two specific families before reading further.} What is the most-expressive family $\K$ for which we can estimate $\cdl_\K$ efficiently?
    \item {\bf Efficient post-processing.} If a predictor has large calibration decision loss relative to $\K$, how should we post-process it?
    \item {\bf Relation to weight-restricted calibration.} How does calibration decision loss relative to $\K$ relate to weight-restricted calibration? Does small calibration error in one sense imply small error in the other?
\end{enumerate}

The main contribution of this paper is to develop a comprehensive theory of when calibration decision loss relative to $\K$ is tractable. We use this theory to answer the questions raised above. Some highlights from our results include:
\begin{itemize}
    \item We show that allowing  all Lipschitz post-processings results in a $\cdl$ notion that is information-theoretically hard to estimate, requiring unbounded sample size. This is in stark contrast to weighted calibration, where allowing all Lipschitz weight functions yields smooth calibration, which is not only tractable but captures distance to calibration and anchors the notion of consistent calibration measures as defined by \cite{BlasiokGHN23}.
    \item We show that restricting post-processings to monotone functions and their generalizations yield  tractable notions of $\cdl$. Our theoretical results justify the central role such families play in practice \cite{Platt, ayer1955empirical, zadrozny2002transforming}. 
    \item We establish tight connections between calibration decision loss and weight restricted calibration for any valid post-processing class $\K$, unifying and generalizing the previous results of \cite{kleinberg2023u, HuWu24, RosselliniSBRW25}.
    \item We introduce new algorithmic techniques from the omniprediction literature to post-processing. We give rigorous new guarantees for some commonly used recalibration procedures in the machine learning literature like Isotonic regression/Pool Adjacent Violators \cite{ayer1955empirical, zadrozny2002transforming} and Uniform mass binning  \cite{zadrozny2001obtaining, gupta2021distribution,sun2023minimum}.
\end{itemize}

\section{Overview of Our Results}

In this section, we present the key definitions and an overview of our main results, and highlight some of the key ideas behind them, without getting into the technical details. 

\subsection{Definitions}
\paragraph{Proper loss functions.}

\begin{definition}[Proper Losses]
\label[definition]{definition:proper-losses}
    A loss $\ell:[0,1]\times\{0,1\}\to \R$ is \emph{proper} if for all $p,q\in[0,1]$,
    \[
        \scalebox{1.2}{$\mathbb{E}$}_{y \sim \Be(q)}\bigl[\ell(p,y) - \ell(q,y)\bigr] \ge 0.
    \]
    We let $\L^*$ be the class of proper losses such that $\abs{\ell(p, 1) - \ell(p, 0)} \le 1$ for all $p \in [0, 1]$.\footnote{Note that our definition of $\L^*$ includes all proper loss functions with range $[0,1]$. The exact range ($[-1,1]$ versus $[0,1]$) is not crucial, it will only change the definition of $\cdl$ by a factor of $2$. }
\end{definition}

Of particular importance  are the \emph{V-shaped} losses studied in \cite{HartlineSLW23, kleinberg2023u}, which are all functions
    \[
        \ell_{v}(p, y) = -\sign(p-v)(y-v),
    \]
where $v \in [0, 1]$. These are proper losses, and moreover they form a basis for $\L^*$; a precise formulation which builds on \cite{kleinberg2023u} is given in \cref{theorem:uv-cal}. 

\paragraph{Post-processing functions.} We say that a class of post-processings $\kappa : [0, 1] \to [0, 1]$ is \emph{valid} if it contains the identity function and satisfies a certain \emph{translation invariance} under shifts of either axis (the precise definition is  in \cref{def:valid-post}). Most natural classes of post-processings we know that have been considered previously, including $\Kall$, Lipschitz, monotone and bounded degree polynomial post-processings, are valid classes. For a valid post-processing class, $\K$, let 
\[
        \threshold(\K) = \Bigl\{p \mapsto \sign_+\Bigl(\kappa(p) - \frac{1}{2}\Bigr) \;\Big|\; \kappa \in \K \Bigr\}.
    \] 
While the choice of $1/2$ might seem arbitrary, the translation invariance of $\K$ implies that any constant in $(0,1)$ will do.

We will pay special attention to the class of \emph{monotonic} post-processing functions, as well as a natural generalization of this class. To define the class, recall that an \emph{interval} $I \subseteq [0, 1]$ with no further specification may be open, closed, half-open, or a singleton.

\begin{definition}[Generalized Monotonicity]
\label[definition]{definition:generalized-monotonicity}
    Given an integer $r \in \N$ and a function $\kappa : [0, 1] \to [0, 1]$, we say $\kappa \in \M_r$ if for all values $v \in \R$, the \emph{$v$-superlevel set} of $\kappa$,
    \[
        \kappa^{-1}\bigl([v, 1]\bigr) = \bigl\{p \in [0, 1] \,\big|\, \kappa(p) \ge v\bigr\},
    \]
    can be expressed as the union of at most $r$ disjoint intervals $I_1, \ldots, I_r \subseteq [0, 1]$. Then, $\M_r$ is a valid post-processing class. In addition, let $\M_+$ and $\M_-$ denote the sets of monotonically nondecreasing and nonincreasing functions $\kappa : [0, 1] \to [0, 1]$, respectively.
\end{definition}

We call $\M_r$ a class of ``generalized'' monotonic functions because the monotonic functions $\M_+$ and $\M_-$ are both subsets of $\M_1$, and because $\M_r \subseteq \M_s$ for all $r \le s$. We also note that $\M_r$ and $\M_+$ are valid post-processing classes, but $\M_-$ is not: Although $\M_-$ is translation invariant, it does not contain the identity function. Generalized monotone functions are the broadest class of weight functions that admit efficient algorithms for $\cdl$ by our results. In Section \ref{sec:why-gmr} we discuss why they constitute a natural class of post-processings to consider, both from a theoretical and practical viewpoint. In \Cref{fig:crossings} we give an example of a function which is (very) non-monotone, but is a $3$-generalized monotone function.

\begin{figure}
    \centering
    \includegraphics[width=0.3\linewidth]{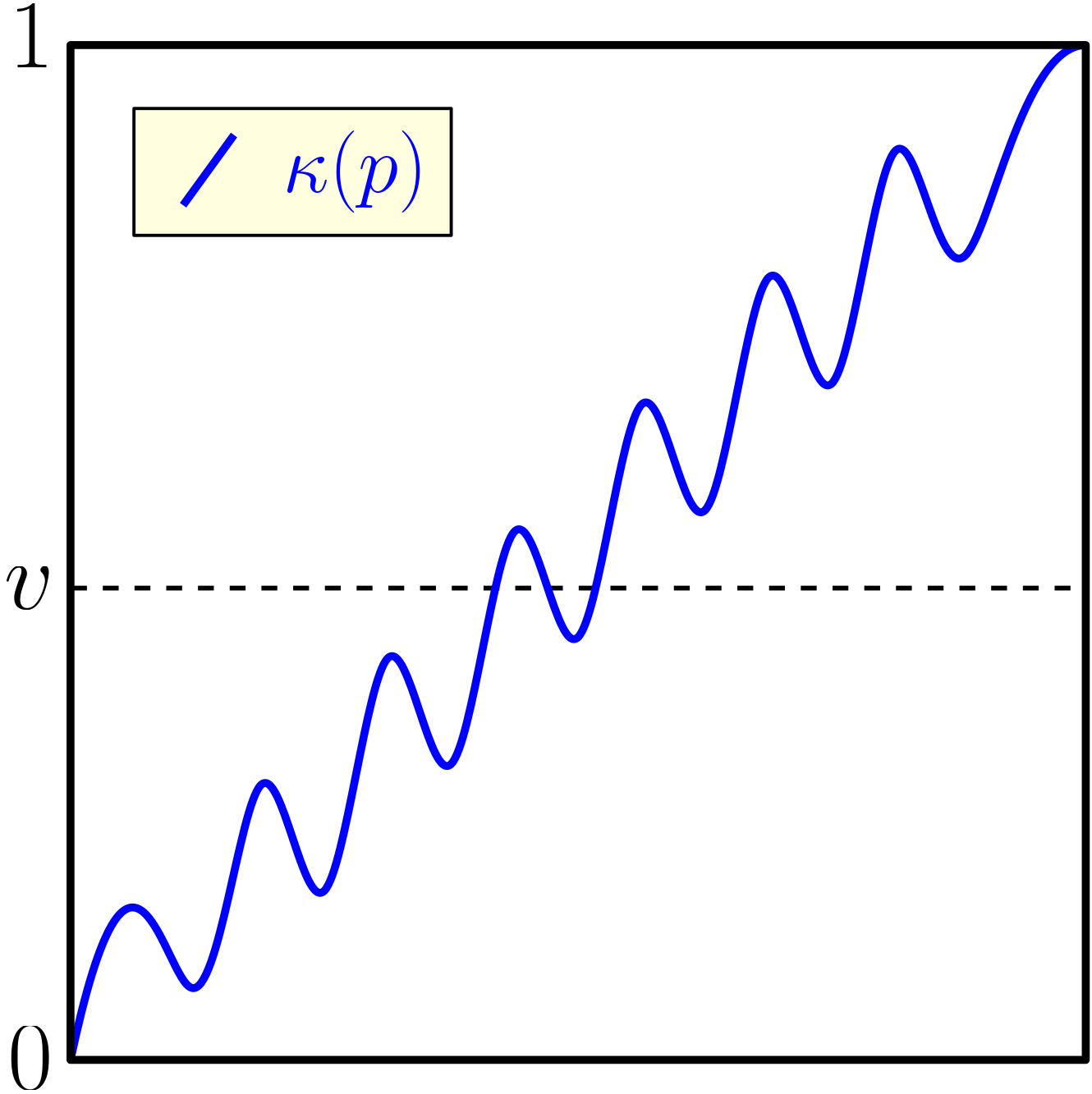}
    \caption{Function $\kappa:[0,1]\to [0,1]$ that crosses every threshold $v\in [0,1]$ of its range at most $3$ times. The monotonicity of the function $\kappa(p)$ changes $14$ times as $p$ grows from $0$ to $1$. Although $\kappa$ is non-monotone, we have $\kappa\in \M_r$ for $r=3$.}
    \label{fig:crossings}
\end{figure}

\paragraph{Calibration Decision Loss.} Next we define \emph{calibration decision loss} ($\cdl$), introduced by \cite{HuWu24}. We study a version of $\cdl$ that is parameterized by a \emph{valid post-processing class $\K$}.

\begin{definition}[Calibration Decision Loss relative to $\K$]
\label[definition]{definition:cdl}
    Given valid post-processing class $\mathcal{K}$ and a joint distribution $J$ over pairs $(p, y) \in [0, 1] \times \{0, 1\}$, the \emph{calibration decision loss} is
    \[
        \cdl_{\K}(J) = \sup_{\ell, \kappa} \, \E\bigl[\ell(p, y) - \ell(\kappa(p), y)\bigr],
    \]
    where the supremum is taken over all $\ell \in \L^*$ and $\kappa \in \K$. 
    Given a subset $\L \subset \L^*$ we define the $(\L, \K)$-calibration decision loss as
     \[
        \cdl_{\L, \K}(J) = \sup_{\ell \in \L, \kappa \in \K} \, \E\bigl[\ell(p, y) - \ell(\kappa(p), y)\bigr].
    \]
    Given a particular $\ell \in \L^*$, we also define a loss-specific version of $\cdl$, called the \emph{calibration fixed decision loss}, as
    \[
        \cdl_{\ell, \mathcal{K}}(J) = \sup_{\kappa} \, \E\bigl[\ell(p, y) - \ell(\kappa(p), y)\bigr].
    \]
    
\end{definition}

For any valid $\K$, $\cdl_\K$ is always non-negative since the identity function is in $\K$.  
When $\K = \Kall$ is the family of \emph{all} functions $\kappa : [0, 1] \to [0, 1]$, then our notion of $\cdl$ coincides with that of \cite{HuWu24}.

In the remainder of this section, we describe the three main tasks related to calibration decision loss that we study in our work: efficient testing/auditing,  relationship to weight-restricted calibration and efficient omniprediction, and our results for each of them.

\subsection{Testing and Auditing}
The first problem we study is the characterization of post-processing classes $\K$ that allow for efficient \emph{testing} of $\cdl_\K$, which we define as follows.

\begin{definition}[Testing $\cdl$]
\label[definition]{definition:testing-cdl}
    We say that an algorithm $\A$ is an $(\alpha,\beta)$-tester for $\cdl_{\K}$, if the algorithm, upon receiving a large enough set of i.i.d. examples from some unknown distribution $J$, outputs either $\accept$ or $\reject$, and satisfies the following properties with probability at least $2/3$.
    \begin{enumerate}
        \item If $\cdl_\K(J) > \alpha$, then $\A$ outputs $\reject$,
        \item If $\cdl_\K(J) \le \beta$, then $\A$ outputs $\accept$.
    \end{enumerate}
\end{definition}

As is typical, the probability here is taken over the sampling process, and the internal randomness of the tester. The tester's output is unconstrained  if $\cdl_\K(J) \in (\beta, \alpha)$.
We also consider a relaxation of the testing problem, called \emph{auditing}, which we define in terms of the \emph{expected calibration error} $\ece(J) = \E \bigl\lvert \E[y|p] - p \bigr\rvert$. Recall that Hu and Wu \cite{HuWu24} showed that $\cdl_\K(J) \le 2 \cdot \ece(J)$. The following definition weakens the second requirement to accepting predictors with small $\ece$.

\begin{definition}[Auditing for $\cdl$]
\label[definition]{definition:auditing-cdl}
    We say that an algorithm $\A$ is an $(\alpha,\beta)$-auditor for $\cdl_{\K}$, if the algorithm, upon receiving a large enough set of i.i.d. examples from some unknown distribution $J$, outputs either $\accept$ or $\reject$, and satisfies the following property, with probability at least $2/3$.
    \begin{enumerate}
        \item If $\cdl_\K(J) > \alpha$, then $\A$ outputs $\reject$,
        \item If $\ece(J) \le \beta/2$, then $\A$ outputs $\accept$.
    \end{enumerate}
    \end{definition}
   
\paragraph{Sample complexity of testing and auditing.}
   Our first result shows that the VC dimension of $\threshold(\K)$ governs the sample efficiency of calibration testing and auditing for $\cdl_\K$. 
   
\begin{theorem}[Sample Complexity Bounds]\label[theorem]{theorem:sample-complexity=intro}
    Let $\K$ be a valid post-processing class, and let $d = \vc(\threshold(\K))$. Then, 
    \begin{enumerate}
        \item For any $\alpha,\eps\in(0,1)$, there is an $(\alpha,\alpha-\eps)$-tester for $\cdl_{\K}$ with sample complexity $\tilde{O}(d/\eps^2)$.
        \item Any $(1/8 , 0)$-auditor for $\cdl_{\K}$ requires $\Omega(\sqrt{d})$ samples.
    \end{enumerate}
\end{theorem}

Some  implications of this result:
\begin{itemize}
\item Auditing is easier than testing, since any $(\alpha, \beta)$-tester for $\cdl_\K$ is also an $(\alpha, \beta)$-auditor. Hence our lower bound for auditing also applies to testers. The case $\beta =0$ corresponds to perfect calibration, so our lower bound holds even for algorithms that must only accept perfectly calibrated predictors, and reject predictors with large $\cdl_\K$. 

\item  A key corollary is that auditing $\cdl_{\Lip}$ is not possible with finitely many samples. This follows from the observation that $\threshold(\Lip)$ has unbounded VC dimension (see \cref{corollary:lipschitz-impossibility}). 

\item The lower bound makes use of V-shaped losses. These losses are not strongly convex unlike common proper losses used in practice, and one may wonder if the lower bound can be circumvented by assuming strong convexity. The answer is No: we prove that essentially the same lower bound carries over even for strongly convex losses (see \Cref{thm:test-audit-sc}).

\item V-shaped losses are discontinuous in the prediction value $p$. If we restrict our attention to loss functions that are Lipschitz continuous as a function of $p$, and only consider the CDL for such losses and Lipschitz post-processings, then we show that this measure is quadratically related to the smooth calibration error, hence it is efficiently auditable (see \cref{thm:cdl-smooth}). 
\end{itemize}

\paragraph{Computationally efficiency from learnability.} We show that efficient algorithms for testing and auditing $\cdl_\K$ can be derived from efficient learning algorithms for $\threshold(\K)$.

We will use the standard primitive of agnostic learning, first introduced by \cite{KearnsSS94}. It is known to be equivalent to  seemingly weaker primitives like weak agnostic learning  \cite{feldman2009distribution, kk09}.

\begin{definition}[Agnostic Learning]\label[definition]{definition:agnostic-learning}
    Let $\C\subseteq\{[0,1]\to \cube{}\}$. We say that an algorithm $\A$ is an $\eps$-agnostic  learner for $\C$ if, upon receiving a large enough set of i.i.d. examples from some unknown distribution $D$ over $[0,1]\times\cube{}$, $\A$ outputs some hypothesis $h:[0,1]\to\cube{}$ such that, with probability at least $0.9$ over the samples and the internal randomness of $\A$, we have:
    \[
        \pr_{(p,z)\sim D}[h(p)\neq z] \le \min_{f\in \C}\pr_{(p,z)\sim D}[f(p)\neq z] + \eps
    \]
    Moreover, if $h \in \C$ we say that $\A$ is proper.
\end{definition}

Our main algorithmic result is a reduction from $\cdl_\K$ testing to proper agnostic learning  for $\threshold(\K)$. For $\cdl_\K$ auditing, improper agnostic learning suffices.

\begin{theorem}[Testing and Auditing from Agnostic Learning, \Cref{theorem:testing-through-proper-agnostic-learning}]\label[theorem]{thm:testing-intro}
    Let $\K$ be a valid post-processing class. Let $\AL$ be an $\eps$-agnostic learner for $\threshold(\K)$. Then, for any $\alpha\in(0,1)$, there is an $(\alpha,\alpha- 3\eps)$-auditor for $\cdl_\K$ that makes $\tilde{O}(1/\eps)$ calls to $\AL$. Moreover, if $\AL$ is proper, then there an $(\alpha, \alpha -3\eps)$-tester for $\cdl_\K$ of similar complexity.
\end{theorem}

This result lets us translate efficient agnostic learning algorithms for classes of Boolean functions on $[0,1]$ to efficient algorithms for $\cdl_\K$ estimation for valid post-processing classes.

\begin{itemize}
\item For the class $\M_+$ of monotone post-processings, $\threshold(\M_+)$ is the family of intervals $[a,1]$, and it is a classic result that intervals are efficiently agnostically learnable. The efficient tester that results from \Cref{thm:testing-intro} strengthens and extends the result of \cite{RosselliniSBRW25} who showed that $\cdl_{\M_+}$ can be bounded by ensuring low weight-restricted calibration error for intervals.

\item  More generally, for  $\kappa \in \M_r$, the sets $\kappa(p) \geq v$ can be expressed as the union of at most $r$ disjoint intervals. We show that this class admits an efficient proper agnostic learner. As a corollary, we obtain an efficient tester for $\cdl_{\M_r}$.  
\end{itemize}

\subsection{Calibration Decision Loss and Weight-Restricted Calibration}

The next problem we study is the relation between $\cdl_\K$ and \emph{weight-restricted} calibration measures, defined below. The standard definition \cite{GopalanKSZ22, GH-survey} either takes the absolute value of the expectation or assumes that $\W$ is closed under negation, we intentionally will do neither. 

\begin{definition}[Weight-Restricted Calibration Error]
\label[definition]{def:weighted-ce}
    Given a distribution $J$ over $(p, y) \in [0, 1] \times \{0, 1\}$ and a class of weight functions $\W  \subseteq \{[0, 1] \to [-1, +1]\}$, the \emph{$\W$-restricted calibration error} is
    \[
        \ce_{\mathcal{W}}(J) = \sup_{w \in \mathcal{W}}\, \E\bigl[w(p) (y-p)\bigr].
    \]
    In the case that $\W$ is closed under negations, we have $\ce_\W(J) \ge 0$.
\end{definition}

Several works \cite{kleinberg2023u, HuWu24, when-does, RosselliniSBRW25} have proved results that can be viewed as instances of this general question for specific choice of loss families $\L$ and post-processings $\K$. Such characterizations are valuable because weight-restricted calibration error measures have been well studied in the literature, and we understand the relation between various measures fairly well \cite{BlasiokGHN23, GH-survey}. Moreover, we would like notions of approximate calibration to simultaneously give small calibration decision loss and strong indistinguishability, making it natural to ask to what extent one implies the other.

We significantly extend and complete this line of work with a general characterization that holds for all valid classes $\K$.

\begin{theorem}
\label[theorem]{thm:cdl-vs-ce-intro}
Given a valid post-processing class $\K$, let
\[
    \threshold'(\K) = \threshold(\K) \cup \Bigl\{p \mapsto -\sign_+(p - v) \;\Big|\; v \in \R \Bigr\}.
\]
    Then \[\ce_{\threshold'(\K)}(J)^2 \lesssim \cdl_\K(J) \lesssim \ce_{\threshold'(\K)}(J).\]
\end{theorem}

While the class $\threshold(\K)$ contains only the \emph{upper} thresholds of functions in $\K$, the class $\threshold'(\K)$ also includes all \emph{lower} thresholds of the identity function. It also contains upper thresholds of the identity function by the inclusion of $\threshold(\K)$, since $\K$ is translation invariant and contains the identity. Since our definition of weight-restricted calibration error did not involve absolute values, the decision for $\threshold'(\K)$ to not include lower thresholds of all functions in $\K$ is deliberate and consequential.

This change in the definition of the appropriate concept class for characterizing testability and for the characterization in terms of weighted calibration error may at first seem counterintuitive. This arises because $\cdl$ is fundamentally an asymmetric notion (meaning we do not assume $\kappa \in \K$ implies $1 - \kappa \in \K$); e.g. monotone increasing post-processings are significantly more natural than monotone decreasing ones. Our definition of $\threshold'(\K)$  precisely captures how much we need to enhance $\threshold(\K)$ to be able to relate it to weighted calibration. Our results on testing, auditing and omniprediction are characterized by naturally symmetric notions such as VC Dimension, and agnostic learning, and would hold equally well if we used $\threshold'(\K)$ instead of $\threshold(\K)$, or even if we added all lower thresholds of functions in $\K$; this would not change the VC dimension  or agnostic learnability. But \Cref{thm:cdl-vs-ce-intro} would no longer hold with $\threshold(\K)$ in place of $\threshold'(\K)$, as explained in the proof sketch below.

This result generalizes and significantly extends results that were previously known in the literature. For instance, when $\K = \Kall$, then $\threshold(\K)$ consists of all functions $\W^* = \{[0,1] \mapsto [-1,1]\}$ and the corresponding weight-restricted calibration notion is just $\ece$. Thus in this case, we recover the result of \cite{HuWu24, kleinberg2023u} which shows that
 \[ \ece(J)^2 \lesssim \cdl_{\Kall} (J) \lesssim \ece(J).\]
 When $\K = \M_+$ is all monotone increasing functions, then $\threshold(\M_+)$ contains all indicators of intervals $[a,1]$ for $a \in [0, 1]$. Let $\intr$ denote the collection of these indicators, along with their complements. In this case, our result shows that
  \[\ce_{\intr}(J)^2 \lesssim \cdl_{\M_+}(J) \lesssim \ce_{\intr}(J).\]
The upper bound was shown in the work of \cite{RosselliniSBRW25}, while the lower bound is new. 
For the class of generalized monotone functions $\M_r$, our theorem shows that $\cdl_{\M_r}$ is quadratically related to a generalized version of $\cdl_{\intr}$, in which single intervals are replaced by unions of at most $r$ intervals.

\paragraph{Proof Sketch.} \cref{thm:cdl-vs-ce-intro} is a highly general result, which transforms a bound on the $\threshold'(\K)$-restricted calibration error into a bound on $\cdl_\K$, and vice versa. In the language of outcome indistinguishability, the key insight underlying the proof is that the two types of functions in the weight class $\threshold'(\K)$, namely upper thresholds of $\K$ and lower thresholds of the identity, are precisely what we need to move, respectively, \emph{to} and \emph{from} the world of simulated outcomes. This perspective builds on the \emph{loss OI} framework of \cite{gopalan2023loss}, particularly their study of \emph{decision OI}.

In slightly more detail, we first consider a weight function obtained by composing a post-processing function $\kappa(p)$ with a proper loss function's \emph{negated discrete derivative} $-\partial \ell(p) = \ell(p, 0) - \ell(p, 1)$. In the worst case---that is, the case of a V-shaped loss function---this composite function corresponds exactly to some \emph{upper threshold} of $\kappa$. Calibration with respect to this composite weight function ensures that moving from real outcomes $y$ to simulated outcomes $\tilde{y}$ can only \emph{improve} the loss attained by $\kappa(p)$, up to an $\eps$ slack factor:
\[
    \E\Bigl[\ell\bigl(\kappa(p), y\bigr)\Bigr] \ge \E\Bigl[\ell\bigl(\kappa(p), \tilde{y}\bigr)\Bigr] - \eps.
\]
Next, we observe that in the simulated world, where $\tilde{y} \sim \Ber(p)$, the predictor $p$ is Bayes optimal by definition. In particular, it outperforms $\kappa(p)$:
\[
    \E\Bigl[\ell\bigl(\kappa(p), \tilde{y}\bigr)\Bigr] \ge \E\bigl[\ell(p, \tilde{y})\bigr].
\]
Finally, we will show that weight functions corresponding to \emph{lower thresholds} of the identity allow us to move back to the world of real outcomes, again up to a small slack factor:
\[
    \E\bigl[\ell(p, \tilde{y})\bigr] \ge \E\bigl[\ell(p, y)\bigr] - \eps.
\]
Combining this chain of inequalities, we deduce that $p$ outperforms $\kappa(p)$ \emph{in the real world, as well}. Phrased differently, $\threshold'(\K)$-restricted calibration ensures that $\cdl_\K$ is small.

The proof of the converse implication---that small $\cdl_\K$ implies small $\threshold'(\K)$-restricted calibration error (up to a quadratic gap)---relies on similar ideas, but is much subtler. For simplicity, consider the most fundamental post-processing class for which the result was not previously known: monotonically increasing functions $\K = \M_+$. In this case, we are given that $\cdl_{\M_+} \le \eps$ and are tasked with proving that $\ce_{\intr} \lesssim \sqrt{\eps}$. To do so, we consider an arbitrary interval $I \subseteq [0, 1]$ and break it into $m = O(1/\sqrt{\eps})$ subintervals $I_1, \ldots, I_m$ with roughly equal probability mass under the distribution of predictions. We then show that the calibration error restricted to a particular subinterval $I_j = [a, b]$ can be bounded by the product of its length and mass, plus an $\eps$ slack term. Since each subinterval has mass $\approx 1/m$ and their lengths sum to at most $1$, our total bound becomes roughly $1/m + m\eps = O(\sqrt{\eps})$. We give our full proof of \cref{thm:cdl-vs-ce-intro} in \cref{sec:cdl-vs-ce}.

\medskip

We conclude by observing that some assumption on $\K$ like validity is necessary for the characterization to hold: the characterization does not hold for the class $\M_-$ of non-increasing post-processings. This class  is translation invariant, but does not contain the identity, so it is not valid by our definition.

\subsection{Post-Processing and Omniprediction}

Here we ask the question: if the predictor $P$ suffers large calibration decision loss relative to $\K$, how should we remedy it? We would ideally like to post-process it in an efficient manner that gives guarantees for every loss in $\L^*$, competitive to baselines from the set $\K$. But the issue is that there might be several losses in $\ell \in \L^*$ that witness large calibration decision loss, each with its own post-processing $\kappa = \kappa^\ell$, which might not be good for a different loss.

To circumvent this, we could allow post-processings that need not themselves lie in $\K$. For instance, if we could take $\kappa(p) = \E[y|p]$ to be the perfect recalibration, then we would have a guarantee for all losses. However, this function will be inefficient to compute (it is as hard as estimating $\ece$). Under what conditions can we efficiently find a post-processing that gives guarantees for all losses?

We formulate this as a problem of efficient \emph{omniprediction} \cite{omni, gopalan2023loss}. The notion of omniprediction originating in supervised learning \cite{omni} asks for a predictor that is simultaneously competitive with the best hypothesis from a class $\C$ of hypotheses, for any loss from a family $\L$. The power of this notion comes from the fact that the best hypothesis in $c$ can depend on the loss $\ell \in \L$, whereas our predictor is oblivious to $\ell$. 
In our context, the goal is to learn a post-processing function $\widehat{\kappa}$ that outperforms any other in $\K$ with respect to \emph{all} possible decision tasks. This is a departure from its standard definition, where the baseline is a hypothesis class $\C$ of functions on the feature space $\X$, and is similar to its formulation in the recent work of \cite{HartlineWuYangFORC2025}, that shows omniprediction guarantees for smooth calibration. 

\begin{definition}[Omniprediction]
\label[definition]{definition:omnipredictors}
    We say that a function $\hat \kappa:[0,1]\to [0,1]$ is an \emph{$(\eps,\K)$-omnipredictor} for some distribution $J$ over pairs $(p, y)$ if for all $\ell \in \L^*$ and $\kappa \in \K$,
    \[
        \E\bigl[\ell(\hat\kappa(p), y)\bigr] \le \E\bigl[\ell(\kappa(p), y)\bigr] + \eps.
    \]
    We say that  $\A$ learns an $(\eps,\K)$-omnipredictor with probability $1 -\delta$ if, upon receiving a large enough set of i.i.d. samples from some unknown distribution $J$, the algorithm $\A$ outputs an $(\eps,\K)$-omnipredictor for $J$ with probability at least $1-\delta$.
\end{definition}

In this section, we will suppress the dependence on $\eps, \delta$ and say that an algorithm learns a $\K$-omnipredictor if it returns an $(\eps, \K)$-omnipredictor with high probability.\footnote{The results will typically involve a reduction to some other algorithm, whose parameters will be suitably chosen functions of $\eps, \delta$. The formal theorem statements make this dependence explicit.}  We show a range of omniprediction guarantees for various classes $\K$, either by using techniques from the omniprediction literature \cite{gopalan2023loss}, or by a new analysis of well-known recalibration procedures that have been proposed in the machine learning literature \cite{ayer1955empirical, zadrozny2001obtaining,gupta2021distribution,sun2023minimum}. We start with the most general result. 

\paragraph{Omniprediction from agnostic learning.}

For all valid post-processing classes $\K$, we prove that an omnipredictor can be efficiently learnt, under the assumption that $\threshold(\K)$ is agnostically learnable. Thus the same assumption we require for efficient auditing of $\cdl_\K$ is in fact sufficient to ensure omniprediction. 

\begin{theorem}[Omniprediction from Agnostic Learning, Informal Version of \Cref{theorem:omniprediction-from-wAL}]\label[theorem]{theorem:omniprediction-from-wAL-intro}
    For every valid post-processing class $\K$, there is an efficient reduction from learning a  $\K$-omnipredictor to  agnostic learning for $\threshold(\K)$. 
\end{theorem}

We follow the loss OI framework of \cite{gopalan2023loss} which is an indistinguishability-based approach to omniprediction. The key difference between their setting and ours is that they compete against a baseline of hypotheses $\C = \{c:\X \to \pmo\}$ where $\X$ denotes the feature space. Whereas in our calibration setting, we do not have a feature space $\X$, we compete against a baseline of post-processing functions $\kappa(p)$.  Nevertheless, we show how one can adapt their techniques to the setting of calibration to learn omnipredictors efficiently.

\paragraph{Pool Adjacent Violators is an omnipredictor.} 

The Pool Adjacent Violators algorithm \cite{ayer1955empirical} solves the problem of isotonic regression: given samples from $J = (p,y)$, it finds a monotone post-processing of a predictor $p$ that minimizes the square loss among all monotone post-processings. Given a sample of \{$(y_i, p_i)\}$ pairs, it starts from the Bayes optimal predictor on the sample $\kappa(p) = \E[y|p]$, and pools/merges any adjacent pair that violates monotonicity, till it reaches a monotone predictor. We present the algorithm formally in Algorithm \ref{algorithm:pav}. 

Various works have observed that it actually gives guarantees for broader classes of loss functions (including convex proper losses); see \cite{zadrozny2002transforming, Brmmer2013ThePA} and references therein. We will show that it is an omnipredictor for all of $\L^*$ relative to monotone post-processings. This general statement is new to our knowledge. For instance  \cite{Brmmer2013ThePA} shows the result for a subset of scoring rules that they call regular proper scoring rules, defined in terms of certain integrals. This does not capture all proper losses (indeed it is unclear to us what subset they capture), and they use considerably more complex arguments. We present a simple argument that relies on the {\em closer is better} property of proper losses (see \cref{thm:proper-losses-partial-improvement}) which states for a proper loss, moving the prediction $p$ closer to the Bayes optimal $\E[y|p]$ can only help. Formally, we establish the following:

\begin{theorem}[Omniprediction through PAV, Informal Version of \Cref{theorem:omniprediction-through-pav}]\label[theorem]{theorem:omniprediction-through-pav-intro}
    Pool Adjacent Violators (PAV) with sufficiently many samples learns a $\M_+$ omnipredictor.
\end{theorem}

This shows that the class of monotone post-processings admits a {\em proper omnipredictor}: for every distribution $J = (p, y)$, there is a single post-processing $\kappa^* \in \M_+$ with the guarantee that for every proper loss $\ell \in \L^*$, and $\kappa \in \M_+$, 
\[ \E[\ell(\kappa^*(p), y)] \leq \E[\ell(\kappa(p), y)].\]
Other than $\Kall$, this is the only natural class of post-processings we know that has this property.

\paragraph{Omniprediction through Bucketing and Recalibration.}

We next analyze bucketed recalibration through uniform-mass binning. Binning is a long-established technique for measuring calibration \cite{miller1962statistical,sanders1963subjective}. The method of uniform-mass binning was introduced by \cite{zadrozny2001obtaining} as the first binning-based approach not only for measuring calibration, but also for obtaining a calibrated predictor. Rather than choosing equal-width bins, we choose bin boundaries as quantiles, so that every bin has roughly the same mass. It remains empirically competitive to this day \cite{naeini2015obtaining,guo2017calibration,roelofs2022mitigating}, and its calibration properties have been theoretically studied as well \cite{gupta2021distribution,sun2023minimum}.

Informally, uniform-mass binning tries to divide $[0,1]$ into bins that each have probability $\eps$ of containing the prediction $p$. If $p$  has a continuous distribution, we can do this by taking the $\eps$-quantiles as bin  boundaries. For general distributions, we might need to allow {\em singleton} buckets consisting of a single point to account for point-masses that might be larger than $\eps$. This gives a partition of $[0,1]$ into $O(1/\eps)$ buckets that are either singletons, or have probability bounded by $\eps$. 

We show that uniform-mass binning followed by recalibration yields omniprediction with respect to the class of generalized monotone post-processings.

\begin{theorem}[Omniprediction from Uniform-mass binning, Informal Version of \Cref{theorem:omniprediction-through-recal}]\label[theorem]{theorem:omniprediction-through-recal-intro}
    For all $r\ge 1$, Uniform-mass binning (with $O(r^2)$ bins) and recalibration (\Cref{algorithm:recal}) learns an $\M_r$-omnipredictor.
\end{theorem}

The proof proceeds by comparing the bucket-wise recalibration $\hat\kappa$ to any hypothesis in $\kappa \in \M_r$, for a specific V-shaped loss $\ell_v$. We rely on two observations:
\begin{enumerate}
    \item If $\sign(\kappa(p) -v)$ is constant for a bucket $I_j$, then $\hat\kappa$ does at least as well as $\kappa$ for the loss $\ell_v$, up to sampling error (which is small by uniform convergence for intervals). 
    \item If $\sign(\kappa(p) -v)$ is not constant for bucket $I_j$, this means that the graph of $\kappa(p)$ crosses the value $v$ in this bucket, so the bucket is not a singleton.
    \begin{itemize}
        \item By the definition of $\M_r$, this can happen for only $2r$ buckets, which contain the endpoints of the $r$ intervals that constitute the set $\kappa(p) \geq v$. 
        \item Since we use Uniform-mass bucketing, the total probability assigned to these buckets is small, which bounds how much worse $\hat\kappa$ is than $\kappa$.
    \end{itemize}
\end{enumerate}

\subsection*{Summary}

Our work presents a comprehensive theory of the complexity of $\cdl_\K$ for binary classification. It delineates classes of post-processings (e.g. Lipschitz) that are intractable and classes that are efficient (e.g. generalized monotone functions). It proves a tight relationship to weight-restricted calibration. It introduces techniques from the literature on omniprediction for post-processing predictors, and proves rigorous new guarantees for some well-known algorithms used in practice. 

We leave open the question of understanding efficient $\cdl$ for the multiclass setting where the number of labels $k$ grows, as is the case in image classification. It is known that the problem in the weight-restricted setting becomes much harder for large $k$, indeed notions like smooth calibration and distance to calibration require sample complexity $\exp(k)$ \cite{GopalanHR24}. It is a challenging question to formulate efficient and meaningful notions of $\cdl$ for this setting.

\section{Preliminaries}
\label[section]{sec:preliminaries}

In this section, we briefly review relevant concepts related to proper losses, calibration error, post-processing classes, and relevant fundamental concepts from learning theory. Proofs for results in this section appear in \cref{sec:prelim-proofs}.

\paragraph{Mathematical Miscellany} Given scalars $t, a, b \in \R$ with $ a\le b$, we write $[t]_a^b$ to denote the projection of $t$ onto the closed interval $[a, b]$. When referring to an \emph{interval} $I \subseteq \R$ with no further specification, we mean that $I$ may be open, closed, half-open, or a singleton. Phrased differently, an interval $I$ may have the form $[a, b]$, $[a, b)$, $(a, b]$, or $(a, b)$ for some $a \le b$. This includes singleton sets of the form $\{a\}$. We shall sometimes write $f(x) \lesssim g(x)$ to mean that $f(x) = O(g(x))$. Similarly, $f(x) \gtrsim g(x)$ means $f(x) = \Omega(g(x))$. To avoid ambiguity, we will only use this notation when both functions $f(x)$ and $g(x)$ under consideration are nonnegative.

\paragraph{Proper Losses} Since the class $\L^*$ of proper losses will be our focus throughout this work, it will be instructive to review the following standard characterization of the functions that it contains.

\begin{lemma}[Proper Loss Characterization]
\label[lemma]{thm:proper-characterization}
    If $\varphi : [0, 1] \to \R$ is concave and $\varphi' : [0, 1] \to [-1, +1]$ is the derivative of $\varphi$ (or an arbitrary superderivative if $\varphi$ is nondifferentiable), then $\L^*$ contains
    \[
        \ell(p, y) = \varphi(p) + \varphi'(p)(y - p).
    \]
    Conversely, every $\ell \in \L^*$ has this form.
\end{lemma}

When we first defined $\L^*$, we insisted that the \emph{discrete partial derivative} with respect to the binary outcome, namely $\partial \ell(p) = \ell(p, 1) - \ell(p, 0)$, be bounded in absolute value (see \cref{definition:proper-losses}). The characterization in \cref{thm:proper-characterization} shows that this corresponds to a bound on $\varphi'(p) = \partial \ell(p)$. The characterization also leads to the following 
useful lemma, we call the {\em closer is better} lemma which states that shifting one's prediction toward the truth can only improve one's loss. To state it, recall that $ \ell(p, q) = \E_{y \sim \Be(q)}[\ell(p, y)]$.

\begin{lemma}[Closer is better]
\label[lemma]{thm:proper-losses-partial-improvement}
    If $\ell \in \L^*$ and $a \le b \le c$, then $\ell(b, c) \le \ell(a, c)$ and $\ell(b, a) \le \ell(c, a)$.
\end{lemma}

The preceding lemma can be used to show a bound on the range of any $\ell \in \L^*$:

\begin{lemma}[Bound on $\abs{\partial \ell}$ Implies Bound on $\abs{\ell}$]
\label[lemma]{thm:prop-loss-range}
    If $\ell \in \L^*$, then there exists $c \in \R$ such that 
    \[
        c - 1 \le \ell(p, y) \le c + 1
    \]
    for all inputs $(p, y) \in [0, 1] \times \{0, 1\}$.
\end{lemma}

Another consequence of the characterization in \cref{thm:proper-characterization} is the following description of the ``boundary'' of the convex set $\L^*$. Specifically, it is roughly the case that any loss in $\L^*$ is a convex combination of a special type of proper loss functions, which we call \emph{V-shaped} losses.

\begin{definition}[V-Shaped Loss]
    Given a scalar $s \in [-1, +1]$, consider the modified sign function
    \[
        \sign_s(t) = \begin{cases} +1 &\text{if $t > 0$,} \\ s &\text{if $t = 0$,} \\ -1 & \text{if $t < 0$}. \end{cases}
    \]
    The \emph{V-shaped} losses are the functions
    \[
        \ell_{v,s}(p, y) = -\sign_s(p-v)(y-v),
    \]
    where $v \in [0, 1]$ and $s \in [-1, +1]$ are any values. Each loss $\ell_{v,s}$ belongs to the class $\L^*$. We pay particular attention to the V-shaped losses $\ell_v := \ell_{v, 0}$ and $\ell_v^+ := \ell_{v, (+1)}$, and $\ell_v^- := \ell_{v, (-1)}$. We also write $\sign_+ = \sign_{(+1)}$ and $\sign_- = \sign_{(-1)}$ for brevity.
\end{definition}

\begin{lemma}[Modification from \cite{kleinberg2023u}]\label[lemma]{theorem:uv-cal}
    Let $\L_0$ be the set containing the V-shaped losses $\ell_v^+$ and $\ell_v^-$ for all $v \in [0, 1]$, as well as constant functions. Then for all $\ell \in \L^*$ and $\eps > 0$, there exists a convex combination $\ell_0$ of finitely many functions in $\L_0$ such that $\lvert \ell(p) - \ell_0(p) \rvert \le \eps$ at all $p \in [0, 1]$.
\end{lemma}

The preceding lemma can be used to reduce a search over all $\ell \in \L^*$ to a search over a much smaller set. For example, in the original definition of $\cdl$, it suffices to take the supremum over just the set of $\ell_v^+$ and $\ell_v^-$ losses, rather than all of $\L^*$. In fact, the following corollary further shows that it suffices to take the supremum over \emph{just} the set of $\ell_v^+$ losses \emph{or} the set of $\ell_v^-$ losses:

\begin{corollary}\label[corollary]{corollary:uv-cal}
    Given a distribution $J$ over pairs $(p, y)$, post-processings $\K$, and $\L \subseteq \L^*$, let
    \[
        \cdl_{\L,\, \K}(J) = \sup_{\ell, \kappa} \, \E\bigl[\ell(p, y) - \ell(\kappa(p), y)\bigr],
    \]
    where the supremum is taken over all $\ell \in \L$ (not $\L^*$) and $\kappa \in \K$. then
    \[
        \cdl_{\K}(J) = \cdl_{\L^+,\, \K}(J) = \cdl_{\L^-,\, \K}(J),
    \]
    where $\L^+ = \{\ell_v^+ \,|\,v \in [0, 1]\}$ and $\L^- = \{\ell_v^- \,|\, v \in [0, 1]\}$.
\end{corollary}

In \cref{def:weighted-ce}, we gave a broad definition of \emph{weighted calibration error} measures, which is parameterized by a class of weight functions $\W$. The larger the class $\W$, the larger the corresponding weighted calibration error. When $\W$ comprises all functions with range contained in $[-1, +1]$, we arrive at the standard notion of \emph{expected calibration error (ECE)}.

\begin{definition}[Expected Calibration Error]
    Given a distribution $J$ over $(p, y) \in [0, 1] \times \{0, 1\}$,
    \[
        \ece(J) = \E \bigl\lvert \E[y - p \,|\, p] \bigr\rvert.
    \]
\end{definition}

As mentioned previously, \cite{kleinberg2023u,HuWu24} showed that the calibration decision loss $\cdl_{\Kall}$ with respect to the class of \emph{all} post-processing functions $\Kall = \{[0, 1] \to [0, 1]\}$ is quadratically related to $\ece$. In particular, the upper bound also holds for $\cdl_\K$ for any subset $\K \subseteq \Kall$.

\begin{theorem}[\cite{kleinberg2023u,HuWu24}]
\label[theorem]{thm:prior-ece-vs-cdl}
    Given a distribution $J$ over $(p, y) \in [0, 1] \times \{0, 1\}$,
    \[
        \ece(J)^2 \le  \cdl_{\mathrm{\Kall}}(J) \le 2\,\ece(J).
    \]
\end{theorem}

\paragraph{Post-Processing}

Here, we briefly discuss some terminology and facts related to post-processing classes that will be useful in subsequent sections. The first thing we need is a definition of what constitutes a \emph{valid} post-processing class $\K$. For this, our only two requirements are that the $\K$ contains the identity and is \emph{translation invariant}.

\begin{definition}[Valid Post-Processing Class]
\label[definition]{def:valid-post}
    We say that a class of functions $\K \subseteq \{[0, 1] \to [0, 1]\}$ is a \emph{valid post-processing class} if the following two conditions hold:
    \begin{itemize}
        \item \textbf{(Identity)} $\K$ contains the identity function $\kappa(p) = p$,
        \item \textbf{(Translation Invariance)} Fix any $s, t \in \R$. If $\kappa$ belongs to $\K$, then so does the function
        \[
            \kappa_{s,t}(p) = \Bigl[\kappa\bigl([p+s]_0^1\bigr) + t\Bigr]_0^1.
        \]
    \end{itemize}
\end{definition}

The fact that $\K$ contains the identity function means that $\cdl_\K$ will always be nonnegative. Later, we will show that the complexity of calibration decision loss relative to a valid post-processing classes $\K$, is closely related to the complexity of the class $\threshold(\K)$ of its upper thresholds.

\begin{definition}[Upper Thresholds]
\label[definition]{def:thr}
    Given a valid post-processing class $\K$, its \emph{upper thresholds} are
    \[
        \threshold(\K) = \Bigl\{p \mapsto \sign_+\Bigl(\kappa(p) - \frac{1}{2}\Bigr) \;\Big|\; \kappa \in \K \Bigr\}.
    \]
\end{definition}

Although the cutoff value of $1/2$ in \cref{def:thr} may seem arbitrary, its choice is not especially important when working with valid post-processing classes, which are translation invariant.

\begin{definition}[VC Dimension]
\label[definition]{definition:vc-dim}
    A class $\C\subseteq\{[0,1] \to \cube{}\}$ \emph{shatters} $S \subseteq [0, 1]$ if every function from $S$ to $\cube{}$ is the restriction of some function in $\C$. The \emph{VC dimension} of $\C$, denoted $\vc(\C)$, is the size of the largest set it shatters. We say $\vc(\C) = \infty$ if the dimension is unbounded.
\end{definition}

\newcommand{\mcr}{\mathrm{cr}}
\newcommand{\mtv}{\mathrm{tv}}

\section{Sample Complexity of Testing and Auditing}
\label[section]{sec:samp-comp}

In this section, we show that the sample complexity of testing/auditing is characterized  (up to a quadratic gap) by the VC dimension of the class $\threshold(\K)$. We use this to derive a lower bound for $\cdl_\Lip$ where $\Lip$ denotes the family of $1$-Lipschitz post-processings. We then consider the possibility of circumventing this lower bound by consider restricted loss families.

\subsection{A Characterization via VC dimension}

Our main sample complexity characterization is as follows.

\begin{theorem}[Sample Complexity Bounds]\label[theorem]{theorem:sample-complexity}
    Let $\K$ be a valid post-processing class, and let $\vc(\threshold(\K)) = d$. Then, 
    \begin{enumerate}
        \item For any $\alpha,\eps\in(0,1)$, there is an $(\alpha,\alpha-\eps)$-tester for $\cdl_{\K}$ whose sample complexity is $O(d\log(1/\eps)/\eps^2)$.
        \item Any $(1/8 , 0)$-auditor for $\cdl_{\K}$ requires $\Omega(\sqrt{d})$ samples.
    \end{enumerate}
\end{theorem}

The upper bound relies on the following generalization result for CDL. 

\begin{lemma}[Uniform Convergence for CDL]\label[lemma]{lemma:uniform-convergence-cdl}
    Let $\K$ be a valid post-processing class, and let $\vc(\threshold(\K)) = d$. For any $\eps,\delta\in(0,1)$ and any set $S$ of $m  =O(d \log(1/\eps\delta)/\eps^2)$ i.i.d. examples from some distribution $J$ over $[0,1]\times\{0,1\}$, with probability at least $1-\delta$, we have:
    \begin{align*}
        & \sup_{\substack{\ell\in \L^* \\ \kappa \in \K}} \Bigr|\E_{(p,y)\sim J}[\ell(p,y) - \ell(\kappa(p),y)] - \E_{(p,y)\sim S}[\ell(p,y) - \ell(\kappa(p),y)] \Bigr|     
        \le \eps.
    \end{align*}
    In particular, with probability at least $1-\delta$ we have $|\cdl_\K(J) - \cdl_\K(S)| \le \eps$.
\end{lemma}

Given this lemma, the upper bound is straightforward: we estimate $\cdl_\K(S)$ over a sufficiently large set of samples $S$, and decide to accept or reject by thresholding at $\alpha - \eps/2$. We defer both the proof of \Cref{lemma:uniform-convergence-cdl} and the derivation of the upper bound in \Cref{theorem:sample-complexity} to \Cref{app:samp-comp}. 

We now show the lower bound for auditing. Let $p_1, p_2, \dots, p_d \in [0,1]$ such that $(\sign_+(\kappa(p_i) - 1/2))_{i\in[d]}$ takes all the possible values in $\cube{d}$ for different choices of $\kappa \in \K$. By the translation invariance of $\K$, we can assume that there exist $q_1, \ldots, q_{d/2} \in [1/4, 3/4]$ such that $(\sign_+(\kappa(q_i) -1/2))_{i\in[d/2]}$ takes all the possible values in $\cube{d/2}$ for different choices of $\kappa\in \K$.\footnote{If at least $d/2$ points lie in $[0,1/2]$, then we consider $q = [p + 1/4]$ and $\kappa'(q) = \kappa([q - 1/4]) = \kappa(p)$. Else, at least $d/2$ points lie in $[1/2,1]$, so we consider $q_i = [p_i -1/4]$ and $\kappa'(q) =\kappa([q + 1/4]) = \kappa(p_i)$.}

We consider the following distributions.
    \begin{enumerate}
        \item Let $J_0$ be a distribution over $[0,1]\times\{0,1\}$ whose marginal on $[0,1]$ is the uniform distribution over the set $\{q_1,\dots,q_{d/2}\}$, and $\E_{(p,y)\sim J_0}[y | p=q_i] = q_i$. 
        \item Let $\J_1$ be a distribution over distributions that is defined as follows. To sample a distribution $J_1$ from $\J_1$, we draw $d/2$ independent random variables $y_i\sim \Be(q_i)$ for $i = 1,\dots,d/2$. The marginal of $J_1$ on $[0,1]$ is uniform over $\{q_1,\dots,q_{d/2}\}$, and $\E_{(p,y)\sim J_1}[y | p = q_i] = y_i$.
    \end{enumerate}
    It is easy to show that $J_0$ is perfectly calibrated, whereas every distribution $J_1 \in \J_1$ has $\ece(J_1) \geq 1/4$. This is the basis of a lower bound for estimating the ECE in \cite{GopalanHR24}. 
    It is not true that $\cdl_\K(J_1)$ is large for every choice of $J_1 \in \J_1$. Nevertheless, 
    we will show that $\cdl_\K(J_1)$ is likely to be large for a random $J_1$, which is sufficient for the lower bound to go through.

    \begin{lemma}
    \label[lemma]{lem:main-dist}
        We have
        \[ \Pr_{J_1 \sim \J_1}\lt[\cdl_\K(J_1) \geq \fr{8}\rt] \geq \frac{5}{6}.\]
    \end{lemma}
    \begin{proof}
     Consider the proper loss function 
    \[ \ell(p,y) = \ell_{1/2}^+(p,y) = -(y-1/2)\sign_+(p-1/2).\] 
    We will lower bound $\cdl_\K$ by
    \begin{align}
        \cdl_{\ell, \K}(J_1) &= \sup_{\kappa\in\K} \lt(\frac{2}{d}\sum_{i=1}^{d/2} \lt(y_i - \fr{2}\rt)\sign_+\lt(\kappa(q_i) - \fr{2}\rt) - \sign_+\lt(q_i - \fr{2}\rt)\rt)\notag\\
        &= \underbrace{ \sup_{\kappa\in\K} \lt(\frac{2}{d}\sum_{i=1}^{d/2}\lt(y_i - \fr{2}\rt)\sign_+\lt(\kappa(q_i) - \fr{2}\rt)\rt)}_{a_1} - \underbrace{\frac{2}{d}\sum_{i=1}^{d/2}\sign_+\lt(q_i - \fr{2}\rt)}_{a_2}  \label{eq:cdl-sp}
    \end{align}
    To lower bound $a_1$, we choose $\kappa\in \K$ such that 
    \[ \forall i, \ \sign_+\lt(\kappa(q_i) - \fr{2}\rt) = \sign_+\lt(y_i-\fr{2}\rt).\] 
    There exists such a $\kappa$ for any realization of $(y_i)_i$ due to the choice of $(q_i)_i$.\footnote{This is in fact the $\kappa$ that maximizes $a_1$.} Since $y_i \in \zo$, this ensures that
    \begin{align} 
    \label{eq:bound-a1}
    a_1 \geq \frac{2}{d}\sum_{i=1}^{d/2} \lt(y_i - \fr{2}\rt) \sign_+\lt(\kappa(q_i) - \fr{2}\rt) = \frac{2}{d}\sum_{i=1}^{d/2} \lt|y_i - \fr{2}\rt|  = \frac{1}{2}.
    \end{align}
    Next we show an upper bound on $a_2$. Over the random choice of $y_i\sim \Be(q_i)$, since $\E[y_i|q_i] = q_i$, 
    \begin{align*}
         \E[a_2] &= \frac{2}{d}\sum_{i=1}^{d/2} \lt(q_i - \fr{2}\rt) \sign_+\lt(q_i - \fr{2}\rt)\\
                &= \frac{2}{d}\sum_{i=1}^{d/2} \lt|q_i - \fr{2}\rt|\\
                & \leq \fr{4}
    \end{align*}
    where the last step uses the fact that $q_i \in [1/4, 3/4]$.
    By the Hoeffding bound, for $d$ larger than some constant, $\Pr[a_2 \geq  3/8] \leq 1/6$.   Therefore, with probability $5/6$, by Equations \eqref{eq:cdl-sp} and \eqref{eq:bound-a1}, 
    \[ \cdl_{\K}(J_1) \geq \fr{2} - \frac{3}{8} = \fr{8}.\]
    \end{proof}

    We now complete the proof of the lower bound in \Cref{theorem:sample-complexity}, using a birthday paradox argument, as in \cite{GopalanHR24}.  
    \begin{proof}[Proof of \Cref{theorem:sample-complexity}, Lower bound.]

    Suppose that we have a $(1/8,0)$-auditor $\A$ for $\cdl_{\K}$ with sample complexity at most $m$. Since the distribution $J_0$ is perfectly calibrated 
    \begin{align*}
        \pr_{S_0\sim J_0^m}[\A(S_0) = \accept] \ge 2/3\,.
    \end{align*}
    On the other hand, for $J_1 \in \J_1$, \Cref{lem:main-dist} implies that $\cdl_{\K}(J_1) \geq 1/8$  with probability $5/6$.
    Conditioned on this event, $\A$ being a $(1/8, 0)$ auditor for $\cdl_\K$, $\A(S_1)$ must reject with probability at least $2/3$. This means that
    \begin{equation}
        \pr_{\substack{J_1 \sim \J_1 \\ S_1 \sim J_1^m }}[\A(S_1) = \accept] \le \frac{1}{3}\cdot \frac{5}{6} + \frac{1}{6} < \frac{1}{2}\,. \label{equation:reject-case}
    \end{equation}
    But this means that the auditor $\A$ satisfies the condition
      \begin{equation}
        \label{eq:diff-accept}
        \Bigr|\pr_{S_0 \sim J_0^{m}}[\A(S_0) = \accept] - \pr_{\substack{J_1 \sim \J_1 \\ S_1 \sim J_1^m }}[\A(S_1) = \accept]\Bigr| \ge 1/6\,,
    \end{equation}
    where the probabilities are over the random variables $S_0, J_1, S_1$ and any potential randomness of $\A$. 
     If we condition on the events that all the elements of $S_0$ and all the elements of $S_1$ are distinct, then the corresponding conditional distributions are identical. Therefore, the total variation distance between the distribution of $S_0$ and the distribution of $S_1$ is bounded by the collision probability, which is smaller than $1/6$ unless $m = \Omega(\sqrt{d})$.
    \end{proof}

\paragraph{Lower bound for Lipschitz post-processings.}

\Cref{theorem:sample-complexity} rules out efficient CDL estimation for the class of Lipschitz post-processings. 
\begin{corollary}\label[corollary]{corollary:lipschitz-impossibility}
    Let $\Lip \subset \Kall$ denote the class of $1$-Lipschitz post-processings. There is no $(1/8, 0)$-auditor for $\cdl_{\Lip}$ with finite sample complexity.
\end{corollary}
This follows because  $\threshold(\Lip)$ has infinite VC dimension, this can be seen taking $S$ to be a grid on the interval $[0,1]$ of multiples of $2\gamma$, and observing that the functions which take values in $(1/2 \pm \gamma)$ on the grid points and interpolate linearly between them are Lipschitz. The thresholds of these functions shatter the set $S$. We now take $\gamma \to 0$.

\paragraph{Upper bound for Generalized monotone post-processings.}
\Cref{theorem:sample-complexity} implies a sample-efficient algorithm to estimate $\cdl_{\M_r}$ for the class $\M_r$ of generalized monotone functions. This follows from \cref{proposition:properties-of-generalized-monotone} which bounds their VC dimension. 

\begin{corollary}\label[corollary]{corollary:gm-upper}
    For all $r\geq 1$, $\alpha \in [0,1]$ and $\eps < \alpha$,  there is an $(\alpha, \alpha - \eps)$-tester for $\cdl_{\M_r}$ with sample complexity $O(r\log(1/\eps)/\eps^2)$.
\end{corollary}

This bound by itself does not guarantee computational efficiency. But it can be made computationally efficient, as shown in  \cref{theorem:test-audit-generalized-monotone}.

\subsection{Other Families of Loss Functions}

To conclude this section, we briefly discuss the role played by the class $\L^*$ in our auditing lower bound. For example, one might ask whether our auditing lower bound continues to hold if we replace the class $\L^*$ with natural subclasses that exclude the $\ell_{1/2}^+$ loss, which played a key role in our proof. Here, we will investigate two such relaxations, corresponding to strongly proper loss functions, and loss functions that are Lipschitz in the predictions $p$. Note that V-shaped losses do not satisfy either of these conditions.

\paragraph{Lower bounds for strongly proper losses.}

We consider the class $\L^*_{\mu\mathsf{-sc}}$ of losses $\ell \in \L^*$ whose associated concave function $\varphi(p) = \E_{y \sim \Be(p)} \ell(p, y)$ satisfies \emph{$\mu$-strong concavity} for some $\mu > 0$. To state the definition, recall that we write $\ell(p, q) = \E_{y \sim \Be(q)}\ell(p, y)$.

\begin{definition}
\label[definition]{def:proper-sc}
    We say a function $f : [0, 1] \to \R$ is \emph{$\mu$-strongly concave} if for all $x,y,\lambda \in [0, 1]$,
    \[
        f\bigl(\lambda x + (1-\lambda)y\bigr) \ge \lambda f(x) + (1-\lambda)f(y) + \lambda(1-\lambda) \cdot \frac{\mu}{2}(x-y)^2.
    \]
    We let $\L^*_{\mu\mathsf{-sc}}$ denote the class of \emph{strongly proper} losses $\ell \in \L^*$, meaning that $\varphi(p) = \ell(p, p)$ is $\mu$-strongly concave in $p$. We let $\cdl_{\L^*_{\mu\mathsf{-sc}},\K}$ denote the similarly restricted version of $\cdl_\K$.
\end{definition}

While $\varphi(p)$ is a concave function for any $\ell \in \L^*$, \cref{def:proper-sc} goes beyond this by requiring \emph{strong concavity}. For example, $\varphi(p) = p(1-p)$ is a $\mu$-strongly concave function with $\mu = 2$, corresponding to the squared loss $\ell_{\mathsf{sq}}(p, y) = (y-p)^2$. We now study the testing and auditing of CDL with respect to this restricted class of proper losses, showing that the main results of this section continue to hold.

\begin{corollary}[Sample Complexity for $\L^*_{\mu\mathsf{-sc}}$]
\label[corollary]{thm:test-audit-sc}
    Let $\K$ be a valid post-processing class, and let $\vc(\threshold(\K)) = d$. Then, the following are true.
    \begin{enumerate}
        \item For any $\alpha,\eps, \mu\in(0,1)$, there is an $(\alpha,\alpha-\eps)$-tester for $\cdl_{\L^*_{\mu\mathsf{-sc}},\K}$ whose sample complexity is $O(d \log(1/\eps)/\eps^2)$.
        \item For any $\mu \in (0, 1/16)$, any $(1/8 -2\mu, 0)$-auditor for $\cdl_{\L^*_{\mu\mathsf{-sc}},\K}$ requires $\Omega(\sqrt{d})$ samples.
    \end{enumerate}
\end{corollary}
We defer the proof to \Cref{app:samp-comp}.

\paragraph{Lipschitz losses and Smooth Calibration Error.}\label{section:appendix-smoothed}

\Cref{corollary:lipschitz-impossibility} shows that $\cdl_{\K}$ is intractable when $\K$ consists of Lipschitz post-processings. If we further restrict the proper losses to be Lipschitz, we will show that CDL is tightly characterized by the smooth calibration error \cite{kakadeF08, BlasiokGHN23}, which is known to be both information-theoretically and computationally tractable.  

We define our families of Lipschitz losses and post-processings.
\begin{definition}
\label[definition]{definition:smooth-measures}
    Let $\L_{\Lip} \subset \L^*$  denote the family of all losses $\ell$ such that $\ell(p,y)$ is $1$-Lipschitz in the first argument. Let $\Lip(2)$ denote the family of post-processing functions $\kappa :[0,1]\to [0,1]$ that are $2$-Lipschitz.
\end{definition}

We consider $2$-Lipschitz rather than  $1$-Lipschitz post-processings  for technical reasons: in order to capture functions of the form $[p+w(p)]_0^1$, where $w$ is $1$-Lipschitz. This is convenient in order to provide a characterization of $\cdl$ in terms of $\smce$. But note that allowing more post-processings makes the positive result we will show more powerful, moreover the lower bound of \Cref{corollary:lipschitz-impossibility} also holds for the class of $2$-Lipschitz post-processings.

    Recall that the \emph{smooth calibration error} is
    \[
        \smce(J) = \sup_{w} \E[w(p) (y-p)],
    \]
    where the supremum is taken over all $w:[0,1]\to [-1,1]$ that are $1$-Lipschitz.

\begin{theorem}
\label[theorem]{thm:cdl-smooth}
    For any distribution $J$ over $[0,1]\times \{0,1\}$, we have that:
    \[
        \frac{1}{2}(\smce(J))^2 \le \cdl_{\L_{\Lip}, \Lip(2)}(J) \le 6\cdot \smce(J)
    \]
\end{theorem}

We present the proof in \Cref{app:samp-comp}. The upper bound follows from the loss OI lemma of \cite{gopalan2023loss} (\Cref{thm:loss-oi}), while the lower bound follows from arguments in \cite{when-does}. Since $\smce$ can be estimated efficiently from samples \cite{BlasiokGHN23}, we have an efficient $(\alpha, c\alpha^2)$-auditor for $\cdl_{\L_{\Lip}, \Lip(2)}$. We note that the restriction to Lipschitz losses is a significant one, which does exclude important losses. For instance, consider the $\ell_1$ loss: $\ell_1(p, y) = |p -y|$. It is not proper, but it is Lipschitz in $p$. If we convert it to a proper loss by composing it with the best response, we get the $\ell_{1/2}$ loss, which is not Lipschitz in $p$. This happens because the best-response $\bm{1}[p \geq 1/2]$ is non-Lipschitz. This is not uncommon, even when the original loss is Lipschitz, since decision making in both theory and practice often involves sharp thresholds.

\section{Relation to Weight-Restricted Calibration}
\label[section]{sec:cdl-vs-ce}

In this section, we establish a tight relationship between $\cdl_\K$, the calibration decision loss for a class $\K$, and a certain weight-restricted calibration error measure, a notion we defined in \cref{def:weighted-ce}. The particular weight class we consider will involve thresholds of $\K$.

\subsection{The Characterization}

In order to state our general characterization, recall from \cref{def:thr} that $\threshold(\K)$ denotes the class of \emph{upper thresholds} of a valid post-processing class $\K$. In this section, we will require a slight modification to the class $\threshold(\K)$:

\begin{definition}[Modified Thresholds]
\label[definition]{def:modified-thr}
Given a valid post-processing class $\K$, let
\[
    \threshold'(\K) = \threshold(\K) \cup \Bigl\{p \mapsto -\sign_+(p - v) \;\Big|\; v \in \R \Bigr\}.
\]
\end{definition}

While the class $\threshold(\K)$ contains only the \emph{upper} thresholds of functions in $\K$, the class $\threshold'(\K)$ also includes all \emph{lower} thresholds of the identity function. It also contains upper thresholds of the identity function due to its inclusion of $\threshold(\K)$, since $\K$ is assumed to be translation invariant and contain the identity. The main result of this section is as follows:

\begin{theorem}
\label[theorem]{thm:cdl-vs-ce}
    If $\K$ is a valid post-processing class, then \[\ce_{\threshold'(\K)}(J)^2 \lesssim \cdl_\K(J) \lesssim \ce_{\threshold'(\K)}(J).\]
\end{theorem}

Before discussing the proof of \cref{thm:cdl-vs-ce}, we first present a couple of examples illustrating its use.
Consider $\K = \M_+$, the class of monotonically nondecreasing post-processing functions. In this setting, the theorem implies that $\cdl_{\M_+}$ corresponds (up to quadratic equivalence) to the Interval-restricted calibration error $\propce$ \cite{OKK25, RosselliniSBRW25}---whose weight class consists of all upper and lower thresholds of the identity: $\sign(p - v)$ and $-\sign(p - v)$ for all $v \in [0, 1]$.
The upper bound $\cdl_{\M_+}(J) \lesssim \propce(J)$ recovers a previous result from \cite{RosselliniSBRW25}. The lower bound $\cdl_{\M_+}(J) \gtrsim \propce(J)^2$, however, is new. Moreover, \cref{thm:cdl-vs-ce} is tight: there exist prediction-outcome distributions $J_1$ and $J_2$ such that $\cdl_{\M_+}(J_1) \lesssim \propce(J_1)^2$ and $\cdl_{\M_+}(J_2) \gtrsim \propce(J_2)$. We defer discussion of these examples to \cref{sec:tightness-of-cdl-vs-ce}.

Similarly, given $r \in \N$, we may set $\K = \M_r$, where $\M_r$ is the class of $r$-wise generalized monotone post-processing functions defined in \cref{definition:generalized-monotonicity}. In this case, \cref{thm:cdl-vs-ce} implies that $\cdl_{\M_r}$ is quadratically related to a natural $r$-wise generalization of Interval-restricted calibration error, in which the weight class comprises all indicators for unions of at most $r$ disjoint intervals.

The proof of \cref{thm:cdl-vs-ce} will rely on the following few useful lemmas. The first lemma that we need is an immediate consequence of the bounded convergence theorem. It allows us to work with either strict or weak inequalities interchangeably, since $\lim_{s \to t^+} \bm{1}[p \ge s] = \bm{1}[p > t]$, etc.

\begin{lemma}
\label[lemma]{thm:ce-limit}
    Fix any distribution $J$ over pairs $(p, y) \in [0, 1] \times \{0, 1\}$. Then, any pointwise convergent sequence of weight functions $w_1, w_2, \ldots : [0, 1] \to [-1, 1]$ satisfies
    \[
        \lim_{k \to \infty} \E \bigl[w_k(p)(y-p)\bigr] = \E \biggl[\lim_{k \to \infty}w_k(p)(y-p)\biggr].
    \]
\end{lemma}
The second lemma we need relates a gap in loss values to an expression that resembles weight-restricted calibration. We call it the \emph{decision OI} lemma for its connection to \emph{outcome indistinguishability (OI)} \cite{OI}---specifically, the notion of decision OI from \cite{gopalan2023loss}.
\begin{lemma}[Decision OI]
\label[lemma]{thm:loss-oi}
    For any fixed $p, q \in [0, 1]$ and $y \in \{0, 1\}$,
    \[\ell(p, y) - \ell(q, y) \le \bigl(\partial \ell(p) - \partial \ell (q)\bigr)(y-p).\]
\end{lemma}

\begin{proof}
    Let $\varphi$ be the concave function given by \cref{thm:proper-characterization} such that \(\ell(p, y) = \varphi(p) + \varphi'(p)(y - p),\) where $\varphi' = \partial \ell$ is some superderivative of $\varphi$. Then,
    \[
        \ell(p, y) - \ell(q, y) = \underbrace{\varphi(p) - \varphi(q) - \varphi'(q)(p - q)}_{(*)} + \bigl(\varphi'(p) - \varphi'(q)\bigr)(y - p),
    \]
    and $(*)$ is nonpositive by concavity.
\end{proof}

The third and final lemma we require is more subtle, providing a key relationship between interval miscalibration and the ability of an additive post-processing function to reduce loss. We call it the \emph{small interval lemma}, since we will always apply it to intervals $I$ with either short length $b - a$, or low mass $\Pr[p \in I]$.

\begin{lemma}[Small Interval]
\label[lemma]{thm:small-intervals}
    If $(p, y) \sim J$ and $\K$ contains all functions $\kappa(p) = [p + t]_0^1$ for $t \in \R$, then for all $0 \le a \le b \le 1$  and intervals $I \in \{ [a, b], [a, b), (a, b], (a, b)\}$,
    \[
        \Bigl\lvert \E\bigl[\bm{1}[p \in I] (y - p)\bigr] \Bigr\rvert \le (b-a)\Pr[p \in I] + \frac{1}{2}\cdl_\K(J).
    \]
\end{lemma}

\begin{proof}
By the definition of calibration decision loss,
\[
    \cdl_{\K}(J) \ge \sup_{v,t}\,\E[\ell_v^+(p, y) - \ell_v^+(\kappa_t(p), y)],
\]
where the supremum is taken over all $v \in [0, 1]$ and $t \in [-1, 1]$. This expectation equals
\[
    \E \Bigl[\Bigl(\sign_+\bigl([p+t]_0^1 - v\bigr) - \sign_+(p - v)\Bigr)(y - v)\Bigr] =\begin{cases}
        2 \cdot \E\bigl[\bm{1}[v - t \le p < v](y - v)\bigr] &\text{if }t > 0, \\
        2 \cdot \E\bigl[\bm{1}[v \le p < v - t](v - y)\bigr] &\text{if }t < 0, \\
        0 & \text{if }t = 0.
    \end{cases}
\]
Now consider any $0 \le a < b \le 1$. Taking $v = b$ and $t = b - a > 0$, we see that
\begin{equation}
    \E\bigl[\bm{1}[a \le p < b](y - b)\bigr] \le \frac{1}{2}\cdl_{\K}(J). \label{eq:weighted-1-1}
\end{equation}
Similarly, taking $v = a$ and $t = a - b < 0$, we see that
\begin{equation}
    \E\bigl[\bm{1}[a \le p < b](a - y)\bigr] \le \frac{1}{2}\cdl_{\K}(J). \label{eq:weighted-1-2}
\end{equation}
We visualize equations \eqref{eq:weighted-1-1} and \eqref{eq:weighted-1-2} in Figure~\ref{fig:weighted-1}.
\begin{figure}
    \centering
    \includegraphics[width=0.6\linewidth]{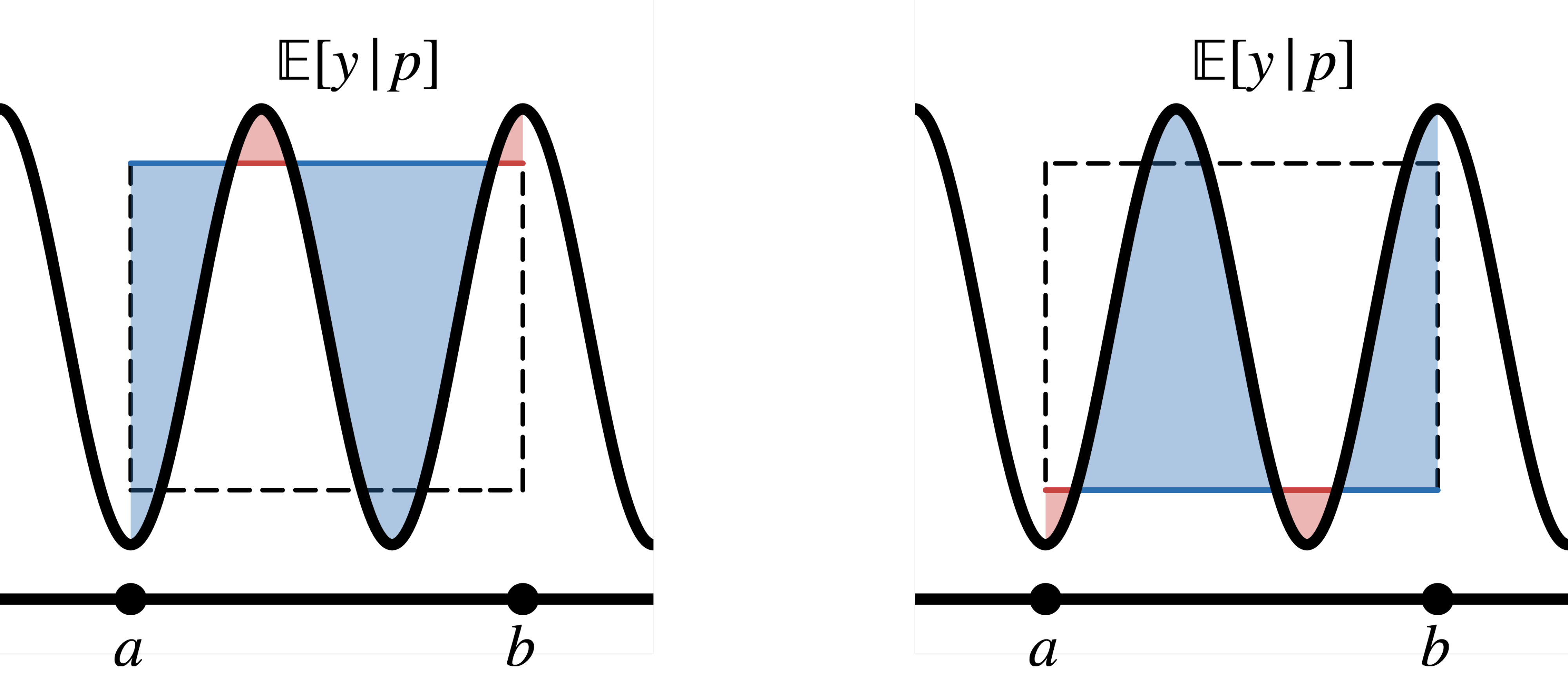}
    \caption{Visualization of inequalities \eqref{eq:weighted-1-1} and \eqref{eq:weighted-1-2}. The dashed line outlines the box $[a, b] \times [a, b]$ in the $(p, y)$-plane. The left and right images correspond to \eqref{eq:weighted-1-1} and \eqref{eq:weighted-1-2}, respectively. The inequalities assert that in each image, the blue area is at least as large as the red, minus a $\cdl_\K(J)$ slack term, assuming the marginal distribution of $p$ is uniform.}
    \label{fig:weighted-1}
\end{figure}
Combining \eqref{eq:weighted-1-1} and \eqref{eq:weighted-1-2} yields
\[
    a \cdot \Pr[a \le p < b] - \frac{1}{2}\cdl_{\K}(J) \le \E\bigl[y \cdot \bm{1}[a \le p < b]\bigr] \le b \cdot \Pr[a \le p < b] + \frac{1}{2}\cdl_{\K}(J).
\]
Next, we subtract the quantity $(a+b)/2 \cdot \Pr[a \le p < b]$ from the left, right, and middle of the above chain of inequalities. Doing so yields
\[
    \biggl\lvert\E\Bigl[\bm{1}[a \le p < b]\Bigl(y - \frac{a+b}{2}\Bigr)\Bigr]\biggr\rvert \le \frac{b - a}{2}\Pr[a \le p < b] + \frac{1}{2}\cdl_{\K}(J).
\]
Observe that, conditional on the event that $a \le p < b$, the midpoint $(a+b)/2$ of the interval $[a, b)$ differs from the prediction $p$ by at most $(b-a)/2$. Consequently, we have
\[
    \Bigl\lvert\E\bigl[\bm{1}[a \le p < b](y - p)\bigr]\Bigr\rvert \le (b - a)\Pr[a \le p < b] + \frac{1}{2}\cdl_{\K}(J).
\]
The left hand side is the absolute value of the calibration error restricted to the indicator function for the half-open interval $I = [a, b)$. Moreover, our argument up to this point works for any $0 \le a < b \le 1$, so that we have proved the claim for any such interval.

By carrying out the preceding argument with $\ell_v^-$ in place of $\ell_v^+$, we can similarly bound the absolute calibration error restricted to half-open intervals of the form $I = (a, b]$, where $0 \le a < b \le 1$.
By varying $a$ and $b$ and applying \cref{thm:ce-limit}, we can also obtain the bound for any interval $I$ of the form $(a, b]$, $(a, b)$, or $[a, b]$, where $0 \le a \le b \le 1$, except for $I = [0, 1]$, but this case is trivial because the right hand side is at least $1$.
\end{proof}

With these three lemmas in hand, we are ready to prove \cref{thm:cdl-vs-ce}.

\begin{proof}[Proof of \cref{thm:cdl-vs-ce}, Upper Bound]
    As a reminder, our goal is to show that if $\K$ is a valid post-processing class, then
    \[
        \cdl_\K(J) \lesssim \ce_{\threshold'(\K)}(J).
    \]
    We first apply the boundary characterization of $\L^*$ from \cref{corollary:uv-cal}, which yields the following upper bound:
    \[
        \cdl_\K(J) \le \sup_{v, \kappa}\, \E[\ell_v^+(p, y) - \ell_v^+(\kappa(p), y)],
    \]
    where the supremum is over $v \in [0, 1]$ and $\kappa \in \K$. Next, we relate the gap in loss values between $p$ and $\kappa(p)$ to an expression resembling weight-restricted calibration using the decision OI lemma (\cref{thm:loss-oi}):
    \[
        \cdl_\K(J) \le \sup_{v, \kappa}\, \E\Bigl[\bigl(\sign_+(\kappa(p)-v)-\sign_+(p-v)\bigr)(y-p)\Bigr],
    \]
    By assumption, both $\sign_+(\kappa(p) - v)$ and $-\sign_+(p - v)$ belong to $\threshold'(\K)$. We conclude that
    \[
        \cdl_\K(J) \le 2 \cdot \ce_{\threshold'(\K)}(J). \qedhere
    \]
\end{proof}

\begin{proof}[Proof of \cref{thm:cdl-vs-ce}, Lower Bound]
    As a reminder, our goal is to show that if $\K$ is a valid post-processing class, then
    \[
        \ce_{\threshold'(\K)}(J)^2 \lesssim \cdl_\K(J).
    \]
    Suppose that $\cdl_\K(J) = \eps$. We will prove that $\ce_{\threshold'(\K)}(J) \lesssim \sqrt{\eps}$. For this, we first recall that there are two types of weight functions in $\threshold'(\K)$. The first kind of weight functions are functions of the form
    \[
        w_1(p) = \sign_+(\kappa(p) - v) = \bm{1}[\kappa(p) \ge v] \cdot 2 - 1,
    \]
    for some $\kappa \in \K$ and $v \in [0, 1]$. The second kind are functions of the form
    \[
        w_2(p) = -\sign_+(p - v) = \bm{1}[p < v] \cdot 2 - 1,
    \]
    for some $v \in [0, 1]$. Therefore, it suffices to prove that the (signed) calibration errors with respect to all weight functions of the form $p \mapsto \bm{1}[\kappa(p) \ge v]$ or $p \mapsto \bm{1}[p < b]$ and $p \mapsto -1$ are either negative, or positive but small (i.e. $O(\sqrt{\eps})$). We shall do so in the following two lemmas, \cref{thm:upper-ce-bound,thm:abs-int-ce-bound}. When combined, these two lemmas complete the proof of \cref{thm:cdl-vs-ce}.
\end{proof}
\begin{lemma}
\label[lemma]{thm:upper-ce-bound}
    Let $J$ be a distribution over pairs $(p, y) \in [0, 1] \times \{0, 1\}$, and let $\K$ be a valid post-processing class. There exists an absolute constant $C > 0$ such that for all $\kappa \in \K$ and $v \in [0, 1]$,
    \[
        \E\bigl[\bm{1}[\kappa(p) \ge v] (y - p)\bigr] \le C\sqrt{\cdl_\K(J)}.
    \]
    (Note that the expectation can be negative, with arbitrarily large magnitude.)
\end{lemma}
\begin{lemma}
\label[lemma]{thm:abs-int-ce-bound}
    Let $J$ be a distribution over pairs $(p, y) \in [0, 1] \times \{0, 1\}$, and let $\K$ be a valid post-processing class. There exists an absolute constant $C > 0$ such that for all intervals $I \subseteq [0, 1]$,
    \[
        \Bigl\lvert \E\bigl[\bm{1}[p \in I] (y - p)\bigr] \Bigr\rvert \le C\sqrt{\cdl_\K(J)}.
    \]
\end{lemma}

Of the two lemmas, \cref{thm:abs-int-ce-bound} has the slightly simpler proof, so we present it first.

\begin{proof}[Proof of \cref{thm:abs-int-ce-bound}]
    Let $\eps = \cdl_\K(J)$ and let $I \subseteq [0, 1]$ be an interval, which, as usual, may be open, closed, half-open, or a singleton. Let $m \in \N$ be an integer to be specified later. We select a sequence of \emph{cutoff} points $c_1 < \cdots < c_m$ such that $c_1$ and $c_m$ are the left and right endpoints of the interval $I$, and each open interval $(c_{i-1}, c_i)$ contains at most a $1/m$ fraction of the total probability mass of $I$. That is, we require
    \[
        \Pr[c_{i-1} < p < c_i] \le \frac{1}{m} \Pr[p \in I].
    \]
    Because we consider open intervals $(c_{i-1}, c_i)$, such a selection is always possible. Note, however, that the distribution of $p$ may have positive mass on some or all of the cutoff points $c_1, \ldots, c_m$. Next, we split the associated weight-restricted calibration error at the cutoff points $c_i$, as follows:
    \begin{equation}
        \E\bigl[\bm{1}[p \in I](y-p)] = \sum_{i=1}^m \E\bigl[\bm{1}[p = c_i](y-p)\bigr] + \sum_{i=2}^m \E\bigl[\bm{1}[c_{i-1} < p < c_i](y-p)\bigr]. \label{eq:interval-cutoffs}
    \end{equation}
    By the small interval lemma (\cref{thm:small-intervals}) each term the first sum in equation \eqref{eq:interval-cutoffs} is bounded in absolute value by $O(\eps)$, and the $i$\textsuperscript{th} term in the second sum is bounded in absolute value by $(c_i - c_{i-1}) \cdot 1/m + \eps$. Note that the difference $c_i - c_{i-1}$ telescopes when we sum over $i$. Therefore, applying the triangle inequality to equation \eqref{eq:interval-cutoffs}, we have that
    \[
        \Bigl\lvert \E\bigl[\bm{1}[p \in I](y-p)] \Bigr\rvert \le m\eps + (c_m - c_1) \cdot \frac{1}{m} + m\eps.
    \]
    Recall that $c_m - c_1$ is the length of $I$, which is at most $1$. Setting $m = \Theta(1/\sqrt{\eps})$, the above bound reduces to $O(\sqrt{\eps})$, as claimed.
\end{proof}

Next, we present the proof of \cref{thm:upper-ce-bound}, which has a similar structure to that of \cref{thm:abs-int-ce-bound}, but with a few extra steps. This will conclude the proof of \cref{thm:cdl-vs-ce}.
    
\begin{proof}[Proof of \cref{thm:upper-ce-bound}]
    Once again, let $\eps = \cdl_\K(J)$. Instead of considering an interval, we now consider a set of the form $S = \{p \in [0, 1] \mid \kappa(p) \ge v\}$ for some post-processing $\kappa \in \K$ and value $v \in [0, 1]$.
    We again select cutoff points  points $c_0 < \cdots < c_m$ such that $c_0 = 0$ and $c_m = 1$ and each open interval $(c_{i-1}, c_i)$ contains at most a $1/m$ fraction of the mass of $S$. That is, we require $\Pr[p \in S \cap (c_{i-1}, c_i)] \le (1/m)\Pr[p \in S]$. Since we consider open intervals $(c_{i-1}, c_i)$, such a selection is always possible. Note, however, that the distribution of $p$ may have positive mass on some or all of the cutoff points $c_0, \ldots, c_m$. Next, we split the associated weight-restricted calibration error at the cutoff points $c_i$, as follows:
    \begin{equation}
        \E\bigl[\bm{1}[p \in S](y - p)\bigr] = \sum_{i=0}^{m} \E\bigl[\bm{1}\bigl[p \in S \cap \{c_i\}\bigr](y - p)\Bigr] + \sum_{i=1}^m \E\Bigl[\bm{1}\bigl[p \in S \cap (c_{i-1}, c_i)\bigr](y - p)\Bigr]. \label{eq:cdl-vs-ce-decomp}
    \end{equation}
    Our goal will be to argue that each term of the two sums is either negative (with unbounded magnitude) or positive but small (i.e. with magnitude at most $O(\sqrt{\eps})$). First, by the small interval lemma (\cref{thm:small-intervals}), the first sum in equation \eqref{eq:cdl-vs-ce-decomp} is bounded in absolute value by $O(m\eps)$. To bound the $i$\textsuperscript{th} term in the second sum, let $a = c_{i-1}$ and $b = c_i$. We will consider the following three comprehensive cases, showing that the $i$\textsuperscript{th} term in the second sum of equation \eqref{eq:cdl-vs-ce-decomp} is at most $O((b-a)/m + \eps)$.
    \begin{itemize}
        \item \textbf{(Case 1: $S \cap (a, b)$ is an interval)} In this case, \cref{thm:small-intervals} tells us that
        \[
            \Bigl\lvert \E\Bigl[\bm{1}\bigl[p \in S \cap (a, b)\bigr] (y - p)\Bigr] \Bigr\rvert \le (b-a)\Pr[p \in S \cap (a, b)] + \eps,
        \]
        and the right hand side is at most $(b-a)/m + \eps$.
        \item \textbf{(Case 2: $S \cap (a, b)$ is not an interval and $a \notin S$ and $b \in S$)} We first relate the calibration error on $S \cap (a, b)$ to a quantity in which $y-p$ has been replaced with $y-b$, as follows:
        \[
            \E\Bigl[\bm{1}\bigl[p \in S \cap (a, b)\bigr] (y - p)\Bigr] \le \E\Bigl[\bm{1}\bigl[p \in S \cap (a, b)\bigr] (y - b)\Bigr] +  (b - a)\Pr[p \in S \cap (a, b)],
        \]
        and the rightmost term is at most $(b-a)/m$. To bound the first term on the right, which involves a factor of $y - b$, we will relate the entire expectation to the $\ell_b^+$ loss. To do so, recall that $S = \{p \in [0, 1] \,|\, \kappa(p) \ge v\}$ for some $\kappa \in \K$ and $v \in [0, 1]$. By translation invariance, $\K$ also contains the function
        \[
            \kappa'(p) = \Bigl[\kappa\Bigl([p]_a^b\Bigr)+b-v\Bigr]_0^1.
        \]
        We visualize the transformation from $\kappa$ to $\kappa'$ in Figure~\ref{fig:weighted-2}. Indeed, by shifting the graph of $\kappa$ left by $a$ (and back again) and right by $1-b$ (and back again), we can ensure that $\kappa$ takes the constant value $\kappa(a)$ on $[0, a]$ and the constant value $\kappa(b)$ on $[b, 1]$, which corresponds to the truncation $[p]_a^b$ in the formula for $\kappa'$.
        
        Next, we claim that $\kappa'(p) \ge b$ if and only if $p \in (S \cap (a, b)) \cup [b, 1]$. To prove this, recall our assumption that $a \notin S$, which implies $\kappa'(p) < b$ for all $p \le a$. Similarly, since $b \in S$, we have $\kappa'(p) \ge b$ for all $p \ge b$. For $p \in (a, b)$, we have $\kappa'(p) \ge b$ if and only if $p \in S$ as well.
        \begin{figure}
            \centering
            \includegraphics[width=0.6\linewidth]{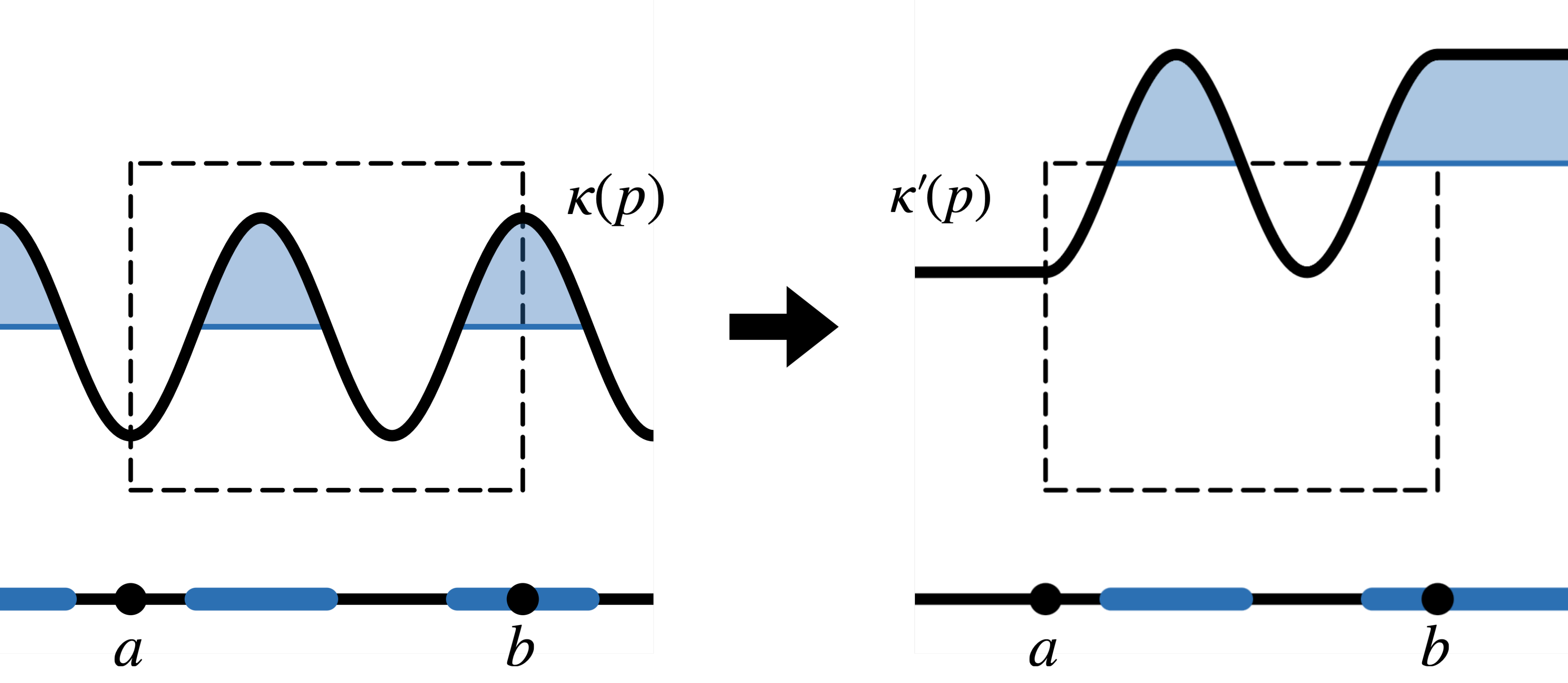}
            \caption{Transformation from $\kappa$ to $\kappa'$. The dashed line outlines the box $[a, b] \times [a, b]$ in the $(p, \kappa(p))$-plane. Starting from the set $S = \{p \in [0, 1] : \kappa(p) \ge v\}$, the transformation removes from $S$ any points to the left of $a$, adds to $S$ any points to the right of $b$, and performs a vertical shift so that the threshold coincides with $b$. The transformation requires $\kappa(a) < v$ and $\kappa(b) \ge v$.}
            \label{fig:weighted-2}
        \end{figure}
        In terms of $\kappa'$, we have
        \[
            \E\bigl[\bm{1}[p \in S \cap (a, b)](y - b)\bigr] = \E\bigl[(\bm{1}[\kappa'(p) \ge b] - \bm{1}[p \ge b])(y - b)\bigr].
        \]
        This can be rewritten in terms of the $\ell_v^+$ loss as
        \[
            \E\bigl[\ell_b^+(p, y) - \ell_b^+(\kappa(p), y)\bigr].
        \]
        By the definition of $\cdl_\K$, this quantity is either negative or at most $O(\eps)$. We conclude that
        \[
            \E\Bigl[\bm{1}\bigl[p \in S \cap (a, b)\bigr] (y - p)\Bigr] \le \frac{b-a}{m} + \eps.
        \]
        
        \item \textbf{(Case 3: $S \cap (a, b)$ is not an interval and either $a \in S$ or $b \notin S$)} In this case, we will perform a simple trick to reduce Case 2, where $a \notin S$ and $b \in S$. Basically, the idea is to shift $a$ to the right until we hit a point not in $S$, and shift $b$ to the left until we hit a point in $S$.
        
        More formally, let $a' = \inf\, (a, b) \setminus S$ be the infimum of points outside $S$ in the interval, and let $b' = \sup\, (a, b) \cap S$ be the supremum of points in $S$ in the interval. Clearly, $a' \ge a$ and $b' \le b$. Also, we have the strict inequality $a' < b'$ since $S \cap (a, b)$ is not an interval.
        
        Next, take any nonincreasing sequence of points $a_j \notin S$ that approach $a$, as well as a nondecreasing sequence of points $b_j \in S$ that approach $b$. Suppose also that we have the strict inequality $a_j < b_j$. To bound the calibration error restricted to $S \cap (a', b')$, we apply the result of Case 2 to each pair $(a_j, b_j)$ and take the limit as $j \to \infty$ (\cref{thm:ce-limit}). This yields
        \[
            \E\Bigl[\bm{1}\bigl[p \in S \cap (a', b')\bigr] (y - p)\Bigr] \le \frac{b'-a'}{m} + \eps.
        \]
        The calibration errors restricted to $S \cap (a, a']$ and $S \cap [b', b)$ are similarly bounded by the small interval lemma (\cref{thm:small-intervals})
        , since the former is a contiguous interval and the latter is either a singleton or empty set. This yields
        \[
            \E\Bigl[\bm{1}\bigl[p \in S \cap (a, a')\bigr] (y - p)\Bigr] \le \frac{a'-a}{m} + \eps
        \]
        and
        \[
            \E\Bigl[\bm{1}\bigl[p \in S \cap (b', b)\bigr] (y - p)\Bigr] \le \frac{b-b'}{m} + \eps.
        \]
        Adding these three inequalities, we see that, in total, over $S \cap (a, a']$, $S \cap (a', b')$, and $S \cap [b', b)$, we have
        \[
            \E\Bigl[\bm{1}\bigl[p \in S \cap (a, b)\bigr] (y - p)\Bigr] \le \frac{b-a}{m} + 3\eps.
        \]
    \end{itemize}
    At this point, we have shown that the $i$\textsuperscript{th} term of the second summation of equation \eqref{eq:cdl-vs-ce-decomp} is either negative (with arbitrarily large magnitude) or positive but bounded in magnitude by $O((c_i - c_{i-1})/m + \eps)$. Setting $m = \Theta(1/\sqrt{\eps})$, the entire bound telescopes to $O(\sqrt{\eps})$, proving \cref{thm:upper-ce-bound}.
\end{proof}

\subsection{Translation Invariance Alone Is Insufficient}

We have established that $\cdl_\K$ is polynomially characterized by the $\threshold'(\K)$-restricted calibration error for any valid post-processing class $\K$. Here, we provide a example that justifies these assumptions, by showing that the characterization does not hold for the class of non-increasing post-processings. Recall that the class of non-increasing post-processings is translation invariant, but does not contain the identity.
\begin{theorem}\label[theorem]{theorem:cdl-ce-separation}
    Let  $\K = \M_-$ and $\W = \threshold'(\M_-)$. There is a distribution $J$ such that:
    \[  
        \ce_\W(J) \ge \frac{1}{4} \text{ , yet }  \cdl_{\K}(J) \le 0
    \]
\end{theorem}

\begin{proof}
    Let $J$ be the distribution over $(p,y)$ such that the marginal distribution of $p$ is $0$ with probability $1/2$ and $1/2$ with probability $1/2$. Moreover, we have
    \[
        p^*(p) := \E_{(p,y)\sim J}[y | p] = \begin{cases}
            0, \text{ if }p=0 \\
            1, \text{ if }p=1/2
        \end{cases}
    \]
    The set $\W$ contains the constant function $w^* : p\mapsto 1$, corresponding to the nonincreasing constant postprocessing $\kappa: p\mapsto 1$. Therefore, we have the following:
    \begin{align*}
        \ce_\W (J) &\ge \E_{(p,y)\sim J}[w^*(p) (y-p)] \\
        &= \E_{(p,y)\sim J}[y-p] \\
        &= \frac{1}{2}\pr_{(p,y)\sim J}[p=1/2] = \frac{1}{4}\,.
    \end{align*}
    Due to \Cref{corollary:uv-cal}, in order to show that $\cdl_\K(J) \le 0$, it suffices to show that $Q_{v,\kappa}^\star \le 0$ for all $v\in[0,1], \kappa\in \K$ and $\star\in \{+,-\}$, where $Q_{v,\kappa}^\star$ is defined as follows:
    \begin{align*}
        Q_{v,\kappa}^\star :=& \E_{(p,y)\sim J}[(y-v) ( \sign_\star(\kappa(p) - v) - \sign_\star(p-v))] \\
        =&\; \frac{1}{2} (-v)  ( \sign_\star(\kappa(0) - v) - \sign_\star(-v)) + \frac{1}{2} (1-v)  ( \sign_\star(\kappa(1/2) - v) - \sign_\star(1/2-v))\\
        =&\; \frac{1}{2} (-v) (\sign_\star(\kappa(0) - v) +1) + \frac{1}{2} (1-v)  ( \sign_\star(\kappa(1/2) - v) - \sign_\star(1/2-v))
    \end{align*}
    The first term in the above expression is non-positive. Therefore, $Q_{v,\kappa}^\star \le 0$ unless the second term is positive. If the second term is positive, then $\sign_\star(\kappa(1/2) - v) = 1$, which implies, due to monotonicity of $\kappa$, that $\sign_\star(\kappa(0) - v) = 1$. Assuming that $\sign_\star(\kappa(1/2) - v) = \sign_\star(\kappa(0) - v) = 1$, we have the following cases.

    \paragraph{Case I: $v \ge 1/2$.} We have $Q^{\star}_{v,\kappa} \le -v + (1-v) = 1-2v \le 0$.
    \paragraph{Case II: $v< 1/2$.} We have $Q^{\star}_{v,\kappa} = -v \le 0$.
\end{proof}

\begin{remark}
The characterization of \Cref{thm:cdl-vs-ce} fails even if we consider a class $\K$ that includes both $\M_-$ and the identity function, since the identity function does not alter the loss value.
Hence, neither translation invariance nor inclusion of the identity function alone suffices for the characterization of \Cref{thm:cdl-vs-ce}.
\end{remark}

\section{Computationally Efficient Testing and Auditing}

In this section, we show that the problem of testing $\cdl_\K$ for some post-processing class $\K$ can be reduced to proper agnostic learning of the threshold class $\threshold(\K)$, whereas auditing reduces to improper agnostic learning.
Note that the characterization of calibration decision loss in terms of the calibration error with respect to the threshold class we provided in \Cref{sec:cdl-vs-ce} already establishes a connection between agnostic learning and $\cdl$ testing: it says that an agnostic learner for $\threshold(\K)$ gives an $(\alpha, c\alpha^2 -\eps)$ tester for some constant $c$.  However, our testing guarantee here is stronger, since we get $\beta = \alpha -\eps$. 

\begin{theorem}[Testing and Auditing from Agnostic Learning]\label[theorem]{theorem:testing-through-proper-agnostic-learning}
    Let $\K$ be a valid post-processing class. Let $\AL$ be an agnostic $\eps$-learner for $\threshold(\K)$ with sample complexity $m$. For any $\alpha\in(0,1)$, there is an $(\alpha,\alpha-3\eps)$-auditor for $\cdl_\K$ that makes at most $O((1/\eps)\log(1/\eps\delta))$ non-adaptive calls to $\AL$, uses $O((1/\eps^2) \log(1/\eps\delta))$ additional samples, and performs $\tilde{O}(\log(1/\delta)/\eps^3)$ additional operations.
    Moreover, if $\AL$ is proper, then there is an $(\alpha,\alpha-3\eps)$-tester for $\cdl_\K$ with the same specifications.
\end{theorem}

We obtain the following corollary for the class of generalized monotone post-processings, based on a folkore result on agnostic learning unions of intervals in one dimension (see \Cref{section:appendix-generalized-monotone}).

\begin{corollary}\label[corollary]{theorem:test-audit-generalized-monotone}
    For any $r\ge 1$, $\alpha,\eps\in (0,1)$, there is an $(\alpha,\alpha-\eps)$-tester for $\cdl_{\M_r}$ with sample complexity $\tilde{O}(r/{\eps^2})$ and runtime $\tilde{O}( r^2/\eps^3)$.
\end{corollary}

The proof of \Cref{theorem:testing-through-proper-agnostic-learning} is based on the characterization of proper losses in terms of the V-shaped losses (\Cref{theorem:uv-cal}) and an appropriate discretization argument for the parameter $v\in [0,1]$ associated with the family of V-shaped losses. For each (discrete) choice of $v$, we identify a relevant agnostic learning problem and solve it using the agnostic learning oracle. First, we state the version of the result for a fixed choice of $v$ separately, as a lemma.

\begin{lemma}
\label[lemma]{thm:al-implies-cfdl}
    Let $\K$ be a valid post-processing class. Given $v \in [0, 1]$, denote the calibration fixed decision loss with respect to the loss $\ell_v^+$ by
    \[
        \cdl_{v,\K}(J) = \sup_{\kappa \in \K}\, \E[\ell_v^+(p, y) - \ell_v^+(\kappa(p), y)].
    \]
    Let $\AL$ be an agnostic $(\eps,\delta)$-learner for $\threshold(\K)$ with sample complexity $m$. Then, for any $\alpha,\gamma\in(0,1)$, there exists an algorithm that calls $\AL$ once, uses $O((1/\gamma^2) \log(1/\delta))$ additional samples and operations, and with probability at least $1 - \delta$ outputs an estimate $\widehat{\cdl}_{v,\K}(J)$ satisfying
    \[
        \widehat{\cdl}_{v,\K}(J) \in \Bigl[\cdl_{v,\K}(J) - 2\eps - 2\gamma, \,2\cdot\ece(J) + 2\gamma\Bigr].
    \]
    Moreover, if $\AL$ is proper, then we have the stronger guarantee
    \[
        \widehat{\cdl}_{v,\K}(J) \in \Bigl[\cdl_{v,\K}(J) - 2\eps - 2\gamma, \,\cdl_{v,\K}(J) + 2\gamma\Bigr].
    \]
\end{lemma}

\begin{proof}
    Our strategy will be to manipulate the expression for $\cdl_{v,\K}$ into a form amenable to agnostic learning.
    First, we substitute the definition of the loss function $\ell_v^+(p, y) = -\sign_+(p - v)(y - v)$, and use the translation invariance of the class $\K$. This allows us to rewrite the formula for $\cdl_{v,\K}$ as
    \[
        \cdl_{v,\K}(J) = \E\bigl[-\sign_+(p-v)(y-v)\bigr] + \sup_{\kappa \in \K}\, \E\Bigl[\sign_+\Bigl(\kappa(p)-\frac{1}{2}\Bigr)(y-v)\Bigr].
    \]
    By Hoeffding's inequality, the first expectation in the above equation can be directly estimated up to error $\gamma$ with probability at least $1-\delta$ from a sample of $(p, y)$ pairs of size $O(\log(1/\delta)/\gamma^2)$.
    
    We can similarly get a handle on the second term, which includes a supremum over post-processing functions $\kappa$, using a single call to the agnostic learner with concept class $\C = \threshold(\K)$ and labels $z = y - v \in [-1, +1]$. (For discrete labels $z \in \{\pm 1\}$, simply perform randomized rounding that preserves $z$ in expectation.) If the agnostic learner returns a hypothesis $h : [0, 1] \to \{\pm 1\}$, then we have the following guarantee with probability at least $1 - \delta$:
    \[
        \E\bigl[h(p)(y-v)\bigr] \ge \sup_{\kappa \in \K}\, \E\Bigl[\sign_+\Bigl(\kappa(p)-\frac{1}{2}\Bigr)(y-v)\Bigr] - 2\eps.
    \]
    Note also that $h(p) = \sign_+(\kappa(p) - 1/2)$ for some function $\kappa : [0, 1] \to [0, 1]$, which we may assume belongs to the class $\K$ if the agnostic learner is \emph{proper}. Consequently, we have
    \[
        \E\bigl[h(p)(y-v)\bigr] \le \sup_{\kappa}\, \E\Bigl[\sign_+\Bigl(\kappa(p)-\frac{1}{2}\Bigr)(y-v)\Bigr],
    \]
    where the supremum is taken over all $\kappa : [0, 1] \to [0, 1]$ in case of an improper learner, or all $\kappa \in \K$ in the case of a proper learner. Of course, given an additional $O(\log(1/\delta)/\gamma^2)$ samples, we can estimate the quantity $\E[h(p)(y-v)]$ up to error $\gamma$ with probability at least $1 - \delta$.

    At this point, our algorithm makes a single call to $\AL$, uses $O(\log(1/\delta)/\gamma^2)$ additional samples, and with probability at least $1 - \delta$ outputs a scalar estimate $\widehat{\cdl}_{v,\K}(J)$ satisfying
    \[
        \cdl_{v,\K}(J) - 2\eps - 2\gamma \le \widehat{\cdl}_{v,\K}(J) \le \cdl_{v,\Kall}(J) + 2\gamma,
    \]
    where $\Kall$ is the set of all post-processings $\kappa : [0, 1] \to [0, 1]$. By \cref{thm:prior-ece-vs-cdl}, $\cdl_{v,\Kall}(J) \le 2 \cdot \ece(J)$. In the special case that $\AL$ is a proper agnostic learner, we have the following stronger condition, with $\K$ replacing $\Kall$ in the upper bound:
    \[
         \cdl_{v,\K}(J) - 2\eps - 2\gamma \le \widehat{\cdl}_{v,\K}(J) \le \cdl_{v,\K}(J) + 2\gamma. \qedhere
    \]
\end{proof}

Having shown that agnostic learning can be used to understand the calibration fixed decision loss with respect to a particular $\ell_v^+$ loss, we now prove \cref{theorem:testing-through-proper-agnostic-learning}, which shows that proper and improper agnostic learning similarly imply testing and auditing for $\cdl_\K$.

\begin{proof}[Proof of \Cref{theorem:testing-through-proper-agnostic-learning}]
    Suppose momentarily that \cref{thm:al-implies-cfdl} were to hold \emph{simultaneously} for all values $v \in [0, 1]$, rather than one fixed value. Then, the supremum of the estimates $\widehat{\cdl}_{v,\K}$ would lie between $\cdl_\K(J) - 2\eps - 2\gamma$ and $2 \cdot \ece_\K(J) + 2\gamma$. In the special case that $\AL$ is proper, the supremum would also lie below $\cdl_\K(J) + 2\gamma$. For $\gamma < \eps/4$, this would immediately imply $(\alpha, \alpha - 3\eps)$-auditing, or, in the case of a proper agnostic learner, $(\alpha, \alpha - 3\eps)$-testing.
    
    With this in mind, our strategy will be to efficiently obtain these estimates $x_v$ for a sufficiently well-spaced \emph{net} of points $v_1, \ldots, v_t \in [0, 1]$. By ``well-spaced,'' we mean that replacing any $v \in [0, 1]$ with its nearest neighbor in the net should change the value of $\cdl_{v,\K}(J)$ by at most $O(\gamma)$. For this, recall the formula for the calibration fixed decision loss with respect to $\ell_v^+$:
    \[
        \cdl_{v,\K}(J) = \E\bigl[-\sign_+(p-v)(y-v)\bigr] + \sup_{\kappa \in \K}\, \E\Bigl[\sign_+\Bigl(\kappa(p)-\frac{1}{2}\Bigr)(y-v)\Bigr].
    \]
    By inspection of this formula, it is clear that it suffices for every $v \in [0, 1]$ to have a nearest neighbor $v_i$ in the net that is both \emph{close in length} and \emph{close in probability}:
    \begin{itemize}
        \item \textbf{(Close in Length)} $\abs{v - v_i} \le \gamma$,
        \item \textbf{(Close in Probability)} $\Pr[\sign_+(p - v) \neq \sign_+(p - v_i)] \le \gamma$.
    \end{itemize}
    Such a net can be constructed by taking the union of equally spaced points $0, \gamma, 2\gamma, \ldots, 1$ with several independent samples $p_1, \ldots, p_s$ drawn from the distribution $J$. The equally spaced points clearly guarantee closeness in length. Closeness in probability follows from the fact that we can partition the interval $[0, 1]$ into $O(1/\gamma)$ contiguous subintervals, each of which either has probability mass $\Theta(\gamma)$ or is a singleton of mass $\Omega(\gamma)$. Then, as long as $s = {O}(\log(1/\gamma\delta)/\gamma)$, with probability at least $1 - \delta$, every subinterval contains some sampled point $p_i$.
    
    To conclude, we observe that for $\gamma = \Omega(\eps)$, calling our fixed-$v$ algorithm from \cref{thm:al-implies-cfdl} as a subroutine for each point in the net leads to a total of $O((1/\eps)\log(1/\eps\delta))$ calls. By a union bound over the $\delta$ failure probability of each call to the subroutine, we need only $O((1/\eps^2) \log(1/\eps\delta))$ additional samples, in total. Similarly, the total number of additional operations is also $\tilde{O}(\log(1/\delta)/\eps^3)$.
\end{proof}

\section{Omniprediction}
\label[section]{section:omniprediction}

Recall the definition of omniprediction (\Cref{definition:omnipredictors}). In this section, we prove the following omniprediction guarantees:
\begin{enumerate}

\item We show that if $\threshold(\K)$ is agnostically learnable, then one can efficiently learn an $(\eps, \K)$ omnipredictor. This result is proved by adapting the loss OI framework of \cite{gopalan2023loss} to the calibration setting. 

\item We provide a strong omniprediction guarantee for the classical Pool Adjacent Violators algorithm \cite{ayer1955empirical}: it gives a proper omnipredictor for the class of all monotone post-processings.  

\item We also provide an analysis of bucketed recalibration through uniform-mass binning, which has been studied in \cite{zadrozny2001obtaining,gupta2021distribution,sun2023minimum}. We show that it gives an omnipredictor for generalized monotone functions.
\end{enumerate}

\subsection{Omniprediction From Agnostic Learning}\label[section]{section:omniprediction-from-wAL}

We show a general result that reduces omniprediction for $\K$  to agnostic learning for $\threshold(\K)$. Thus under the same computational assumptions that we needed for efficient auditing of $\cdl_\K$, we can get the strong post-processing guarantee of omniprediction.

The following theorem reduces omniprediction with respect to all proper losses and some post-processing class $\K$ to agnostic learning of the class $\threshold(\K)$.

\begin{theorem}[Omniprediction from Agnostic Learning]\label[theorem]{theorem:omniprediction-from-wAL}
    Let $\K$ be a valid post-processing class. Let $\AL$ be an $(\eps/3)$-agnostic learner for $\threshold(\K)$. There is an algorithm that learns an $(\eps,\K)$-omnipredictor with probability $1-\delta$ that calls $\AL$ $O(\log(1/\delta) / \eps^2)$ times and performs $\poly({1}/{\eps}) \log(1/\delta)$ additional operations.
\end{theorem}

 \cite{gopalan2023loss} show that omniprediction follows from a condition called calibrated multiaccuracy.  Below we define a version  of calibrated multiaccuracy that is tailored to our application.

\begin{definition}[Calibrated Multiaccuracy]\label[definition]{definition:calMA}
    Let $\hat\kappa:[0,1]\to [0,1]$ and $\C\subseteq\{[0,1]\to \cube{}\}$. We say that $\hat\kappa$ is $\eps$-calibrated-multiaccurate w.r.t. $\C$ under the distribution $J$, or $\hat\kappa\in\calMA_\C(J,\eps)$, if the following properties hold.
    \begin{enumerate}
        \item (Calibration). $\ece(\hat J) \le \eps$, where $\hat J$ is the distribution of pairs $(\hat \kappa(p), y)$, where $(p,y)\sim J$.
        \item (Multiaccuracy). $|\E_{(p,y)\sim J}[c(p) (y-\hat\kappa(p))]| \le \eps$ for all $c\in \C$.
    \end{enumerate}
\end{definition}

We now prove the following lemma, which states that calibrated multiaccuracy implies omniprediction. 

\begin{lemma}\label[lemma]{lemma:calMA-implies-omniprediction}
    Let $\hat\kappa \in \calMA_\C(J,\eps)$, where $\C = \threshold(\K)$. Then $\hat\kappa$ is a $(2\eps, \K)$-omnipredictor.
\end{lemma}

\begin{proof}
    We will show that for any $\kappa\in \K$ and any $\ell\in \L^*$ the following holds.
    \begin{equation}
        \E_{(p,y)\sim J}[\ell(\hat\kappa(p),y)] \le \E_{(p,y)\sim J}[\ell(\kappa(p),y)] + 2\eps\label{equation:omniprediction-goal}
    \end{equation}
    We define the distribution $\tilda{J}$ to be the distribution $(p, \tilda{y})$ where  $p$ has the same as the marginal distribution as under $J$ and $\tilda{y} \sim \Ber(\hat\kappa(p))$ so that $\E[\tilda{y} | p] = \hat\kappa(p)$.

    For the following, for any $\ell\in\L^*$, we let $\partial\ell:[0,1]\to [-1,1]$ denote the function $\partial\ell(p) = \ell(p,1)-\ell(p,0)$. Note that for any $p\in[0,1]$ and $y\in\{0,1\}$, we have $\ell(p,y) = y\,\partial\ell(p)+\ell(p,0)$.

    We will first bound the quantity $\Delta_1 := \E_J[\ell(\hat\kappa(p),y)]-\E_{\tilda{J}}[\ell(\hat\kappa(p),\tilda{y})]$ for any $\ell\in\L^*$ as follows.
    \begin{align*}
        \Delta_1 &= \E_J[y\, \partial\ell(\hat\kappa(p)) + \ell(\hat\kappa(p),0)] -\E_{\tilda{J}}[\tilda{y}\, \partial\ell(\hat\kappa(p)) + \ell(\hat\kappa(p),0)] \\
        &= \E_J[y\, \partial\ell(\hat\kappa(p)) ] -\E_{\tilda{J}}[\tilda{y}\, \partial\ell(\hat\kappa(p))] \tag{$J,\tilda{J}$ have the same $p$-marginal} \\
        &= \E_J[(y-\hat\kappa(p))\, \partial\ell(\hat\kappa(p)) ] \tag{$\E[\tilda{y} | p] = \hat\kappa(p)$} \\
        &= \E_{(q,y)\sim \hat J}[(y-q)\, \partial\ell(q) ] \tag{$(\hat\kappa(p), y)\sim \hat J$} \\
        &\le \sup_{w:[0,1]\to [-1,1]} \Bigr|\E_{\hat J}[(y-q) w(q)]\Bigr| = \ece(\hat J) \le \eps \tag{$\ece(\hat J) \le \eps$}
    \end{align*}

    On the other hand, since $\ell$ is a proper loss and $\E[\tilda{y} | p] = \hat\kappa(p)$, \Cref{definition:proper-losses} implies that for for any $p,p' \in [0,1]$, $\E_{\tilda{J}}[\ell(\hat\kappa(p),\tilda{y}) | p] \le \E_{\tilda{J}}[\ell(p',\tilda{y}) | p]$. It follows that for any $\kappa \in \K$,  $\E_{\tilda{J}}[\ell(\hat\kappa(p),\tilda{y})] \le \E_{\tilda{J}}[\ell(\kappa(p),\tilda{y})]$. Combining this with the bound on $\Delta_1$ yields
    \begin{equation}
        \E_{(p,y)\sim J}[\ell(\hat\kappa(p),y)] \le \E_{(p,\tilda{y})\sim \tilda{J}}[\ell(\kappa(p),\tilda{y})] + \eps \label{equation:omniprediction-intermediate}
    \end{equation}

    We will now show that the quantity $\Delta_2 := \E_{(p,\tilda{y})\sim \tilda{J}}[\ell(\kappa(p),\tilda{y})] - \E_{(p,y)\sim J}[\ell(\kappa(p),y)]$ satisfies $\Delta_2\le \eps$ for any $\ell\in \L^*$ and $\kappa\in \K$. Combining the bound on $\Delta_2$ with Eq. \eqref{equation:omniprediction-intermediate} implies Eq. \eqref{equation:omniprediction-goal}.

    By \Cref{theorem:uv-cal}, we can write 
    \[ \partial\ell(p) = -\int_{[0,1]} \sign_+(p-v) \,d\mu_\ell^+(v) -\int_{[0,1]} \sign_-(p-v) \,d\mu_\ell^-(v),\] 
    where $\mu_\ell^\pm$ are measures over $[0,1]$ such that $\mu_\ell^+([0,1])+\mu_\ell^-([0,1]) \le 1$. Therefore, using similar manipulations like those for bounding $\Delta_1$, we obtain the following.
    \begin{align*}
        \Delta_2 &= \E_{(p,y)\sim J}[(\hat\kappa(p) - y) \partial\ell(\kappa(p))] \\
            &\le \sup_{v\in[0,1]} \E_{(p,y)\sim J}[(y-\hat\kappa(p)) \sign_+(\kappa(p) - v)]
    \end{align*}
    Since $\K$ is closed under translations, there is $\kappa'\in\K$ such that 
    \begin{align*}
        \Delta_2 &\le \E_{(p,y)\sim J}[(y-\hat\kappa(p)) \sign_+(\kappa'(p) - 1/2)] \\
        &\le \sup_{c\in \C} \Bigr|\E_{(p,y)\sim J}[(y-\hat\kappa(p))\, c(p)] \Bigr| \tag{By the definition of $\C = \threshold(\K)$} \\
        &\le \eps\,, \tag{$\hat\kappa\in \calMA_\C(J,\eps)$}
    \end{align*}
    which concludes the proof.
\end{proof}

The final ingredient of \Cref{theorem:omniprediction-from-wAL} is the following result from \cite{gopalan2023loss} which reduces learning a predictor satisfying calibrated multiaccuracy to agnostic learning.

\begin{theorem}[Calibrated Multiaccuracy \cite{gopalan2023loss}]\label[theorem]{theorem:calMA}
    Let $\C\subseteq\{[0,1]\to \cube{}\}$ and let $\AL$ be an $(\eps/3)$-agnostic learner for $\C$. There is an algorithm that calls $\AL$ $O(\log(1/\delta) / \eps^2)$ times, performs $\poly(1/\eps)\log(1/\delta)$ additional operations, and outputs $\hat\kappa\in \calMA_\C(J,\eps/2)$ with probability at least $1-\delta$.
\end{theorem}

\subsection{Pool Adjacent Violators Is an Omnipredictor}\label[section]{section:pav}

Pool Adjacent Violators is a classical algorithm for recalibration \cite{ayer1955empirical}, which finds a monotone post-processing of a predictor. It starts with a sample of \{$(y_i, p_i)\}$ pairs. It starts from the Bayes optimal predictor $\kappa(p_i) = \E[y_i|p_i]$, and pools/merges any adjacent pair that violates monotonicity into a single interval $I$ where the prediction is the conditional expectation $\E[y|I]$. We present the algorithm formally in Algorithm \ref{algorithm:pav}. 

\begin{algorithm}[ht]

\caption{$\textsc{PoolAdjacentViolators}(S)$}\label{algorithm:pav}
\KwIn{Set $S$ of $m$ pairs of the form $(p,y)$ where $p\in[0,1], y\in\{0,1\}$}
\KwOut{A pair $({\cal O}, \I)$ where ${\cal O} = ((p_1,y_1),\dots,(p_m,y_m))$ is an ordering of the input set $S$, and $\I$ is a sequence of disjoint intervals on $[0,1]$ that cover $P = \{p_1,\ldots,p_m\}$.}
\BlankLine
Let ${\cal O} = ((p_1,y_1),(p_2,y_2),\dots,(p_m,y_m))$ be such that $p_1\le p_2\le \dots\le p_m$\;
Let $t=0$, $\I^{(0)} = (I_1^{(0)}, \dots, I_m^{(0)})$, where $I_i^{(0)} = \{p_i\}$ for $i\in[m]$\; 
Set $\bar y_i^{(0)} = y_i$ for all $i \in [m]$ \;
\While{there is $j^*\in[m-t]$ such that $\bar y_{j^*}^{(t)} \ge \bar y_{j^*+1}^{(t)}$}{
    Merge the sets $I_{j^*}^{(t)}$ and $I_{j^*+1}^{(t)}$, i.e., set $\I^{(t+1)} = (I_j^{(t+1)})_{j\in[m-t-1]}$, where 
    $
        I_j^{(t+1)} = \left.\begin{cases} 
                        I_j^{(t)}, \text{ if }0\le j<j^* \\
                        I_{j^*}^{(t)} \cup I_{j^*+1}^{(t)}, \text{ if }j=j^* \\
                        I_{j+1}^{(t)}, \text{ if }m-t-1 \ge j > j^*
                      \end{cases}\right\}
    $\; 
    Let $\bar y_{j}^{(t+1)} = \frac{1}{|I_j^{(t+1)} \cap P|} \sum_{i:p_i \in I_j^{(t+1)}} y_{i}$ for all $j\in[m-t]$; \Comment{Can be implemented in $O(1)$}\\
    Update $t \leftarrow t+1$;
}
Return $({\cal O},\I) = ({\cal O}, \I^{(t)})$\;
\end{algorithm}

We show the following guarantee for PAV.

\begin{theorem}[Omniprediction through PAV]\label[theorem]{theorem:omniprediction-through-pav}
    For $\eps,\delta\in (0,1)$, Pool Adjacent Violators (PAV) run on  $O({\log (1/ (\eps\delta))}/{\eps^2})$ samples is an algorithm that learns an $(\eps,\M_+)$-omnipredictor with probability $1-\delta$ that runs in time  $\tilde{O}( 1/\eps^2)$.
\end{theorem}

In order to prove \Cref{theorem:omniprediction-through-pav}, we first consider the empirical version of the omniprediction problem and show the following result.

\begin{theorem}[PAV guarantees]\label[theorem]{theorem:pav-guarantees}
    Let $S$ be a set of $m$ pairs of the form $(p,y)$ where $p\in[0,1]$ and $y\in\{0,1\}$. There is an $O(m\log m)$-time algorithm (\Cref{algorithm:pav}) that computes a monotone post-processing $\hat \kappa\in \M_+$ such that for any proper loss $\ell\in \L^*$ we have:
    \[
        \sum_{(p,y)\in S} \ell(\hat \kappa(p), y) = \min_{\kappa\in \M_+} \sum_{(p,y)\in S} \ell(\kappa(p), y)
    \]
\end{theorem}
\

\begin{proof}
    Let $({\cal O} = (p_i, y_i)_{i\in[m]}, \I = (I_j)_{j\in [m']})$ be the output of \Cref{algorithm:pav}. The monotone post-processing $\hat\kappa$ is obtained by setting for each $i\in[m]$: $\hat\kappa(p_i) = \bar y_{I_j}$, where $j$ is such that $p_i\in I_j$, and interpolating linearly between the points $p_i$, in order to preserve monotonicity.

    For any fixed proper loss $\ell$, we will show that $\hat \kappa$ is an optimal monotone post-processing. To this end, we use an exchange argument to prove the correctness of the greedy approach of \Cref{algorithm:pav}. In particular, we will show that at any step $t\in \{0,1,\dots, m-1\}$ of the algorithm, and for any $j^*\in[m-t]$ such that $\bar y_{j^*}^{(t)} \ge \bar y_{j^*+1}^{(t)}$, there is an optimal monotone post-processing $\tilde{\kappa}$ that gives the same value to all the points in the set $I_{j^*}^{(t)} \cup I_{j^*+1}^{(t)}$. Therefore, it is safe to merge $I_{j^*}^{(t)}$ and $I_{j^*+1}^{(t)}$.

\begin{claim}
    Assume that there is an optimal monotone post-processing that is constant on each of the intervals $I_1,I_2,\dots,I_{n}$.
    Let $I_j, I_{j+1}$ be two intervals such that 
    \[ \bar y_j = \E[y|p \in I_j] \geq \E[y|p \in I_{j+1}] = \bar y_{j+1}.\]
    There exists an optimal post-processing $\tilde{\kappa} \in \M_+$ such that $\tk(I_j) = \tk(I_{j +1})$.
\end{claim}

\begin{proof}
    Consider any post-processing function $\kappa \in \M_+$ that is constant on each of the intervals $I_1,\dots,I_n$.  Our goal is to find $\tk \in \M_+$ which does as well as $\kappa$ for any proper loss $\ell \in \L^*$, and where $\tk(I_j) = \tk(I_{j +1})$. On the other intervals, we will have $\tk = \kappa$ and they both suffer the same loss, so we can ignore those intervals. 

    Let us write $\kappa_i = \kappa(I_i), \tk_i = \tk(I_i)$ for $i \in \{j, j+1\}$. By monotonicity, $\kappa_j \leq \kappa_{j+1}$. If they are equal, we can take $\tk = \kappa$, so assume the inequality is strict. In this case, we have $\bar y_j \geq \bar y_{j+1}$ and $\kappa_j < \kappa_{j+1}$. We now consider three collectively exhaustive cases:
    \begin{enumerate}
        \item $\kappa_j < \kappa_{j+1} \leq \bar y_j$. We set $\tk_j = \tk_{j +1} = \kappa_{j+1}$. By \Cref{thm:proper-losses-partial-improvement}, 
        \[  \ell(\tk_j, \bar y_j) = \ell(\kappa_{j+1}, \bar y_j) \leq \ell(\kappa_j, \bar y_j)\]
        so $\tk$ can only improve $\kappa$ on interval $I_j$, while they agree on $I_{j+1}$. 
        
        \item $\bar y_{j+1} \leq \kappa_j < \kappa_{j+1}$.  We set $\tk_j = \tk_{j +1} = \kappa_{j}$. By \Cref{thm:proper-losses-partial-improvement}, 
        \[ \ell(\tk_{j +1}, \bar y_{j+1}) = \ell(\kappa_{j}, \bar y_{j+1}) \leq \ell(\kappa_{j+1}, \bar y_{j+1})\] 
        so $\tk$ can only improve $\kappa$ on interval $I_{j+1}$, while they agree on $I_j$.

        \item $\kappa_j < \bar y_{j+1} \leq \bar y_j < \kappa_{j+1}$. Pick any $a \in [\bar y_{j+1}, \bar y_j]$ and let $\tk_j = \tk_{j +1} = a$. Since $\kappa_j < a \leq \bar y_{j+1}$, by \Cref{thm:proper-losses-partial-improvement}, $\ell(a, \bar y_j)  \leq \ell(\kappa_{j}, \bar y_j)$. Similarly, since $\bar y_{j+1} \leq a < \kappa_{j+1}$, 
        $\ell(a, \bar y_{j+1})  \leq \ell(\kappa_{j+1}, \bar y_{j+1})$. Hence
        $\tk$ can only improve $\kappa$ on both intervals $I_j, I_{j+1}$.
     \end{enumerate}
     This concludes the proof of the exchange argument.
\end{proof}

For a given loss $\ell$, let $\kappa^*$ be an optimal monotone post-processing for the distribution $S$. Since $\I^{(0)}$ is composed of singletons, $\kappa^*$ is trivially constant on its intervals. The claim then implies that each iteration of PAV inductively preserves the property that there is an optimal monotone postprocessing $\kappa$ that is constant on intervals in $\I^{(t)}$. Considering a $\kappa$ that is optimal for $\I$, and let $\kappa_j$ be its value for $p \in I_j$. Denoting by $\bar{y}_j$ the average of $y$ on $I_j$, we  have:
    \begin{align*}
        \E_{(p,y)\sim S}[\ell(\kappa(p), y)] &= \frac 1 m \sum_{j\in[m']} \sum_{i: p \in I_j} \ell(\kappa(p), y_i)\\
        &=  \frac 1 m \sum_{j\in[m']} \sum_{i: p \in I_j} \ell(\kappa_j, y_i)\\
        & \ge  \frac 1 m \sum_{j\in[m']} \sum_{i: p \in I_j} \ell(\bar{y}_j, y_i)\\
        &= \E_{(p,y)\sim S}[\ell(\hat \kappa(p), y)]\,,
    \end{align*}
    where the inequality follows from the fact that $\ell$ is proper (\Cref{definition:proper-losses}) and $\bar{y}_j$ is the average of $y$ on $I_j$. This implies that $\hat{\kappa}$ is an optimal post-processing.

Finally, since  the output of \Cref{algorithm:pav} satisfies $\bar y_{1} < \bar y_{2} < \dots < \bar y_{{m'}}$, it follows that $\hat{\kappa}$ itself is monotone. The claim follows.

\end{proof}

\begin{remark}
    Even though \Cref{theorem:pav-guarantees} is stated for $\L^*$, it actually holds for all proper loss functions; in the case of PAV, the boundedness assumption is only needed for generalization.
\end{remark}

We are now ready to prove \Cref{theorem:omniprediction-through-pav}.

\begin{proof}[Proof of \Cref{theorem:omniprediction-through-pav}]
    It suffices to show that whenever $|S|\ge {C \log(1/\eps \delta)}/{\eps^2}$ (for some sufficiently large constant $C\ge 1$, where $S$ consists of i.i.d. samples from some distribution $J$ over $[0,1]\times\{0,1\}$, then, with probability at least $1-\delta$, the following is true uniformly over all $\kappa,\hat\kappa\in \M_+$, $\ell\in\L^*$:
    \[
        \Biggr| \E_{(p,y)\sim J}[\ell(\hat\kappa(p), y) - \ell(\kappa(p), y)] - \frac{1}{m}\sum_{(p,y)\in S}(\ell(\hat\kappa(p), y) - \ell(\kappa(p), y)) \Biggr| \le \eps
    \]
    The above inequality follows by combining the fact that $\vc(\threshold(\M_+)) = 1$ with \Cref{lemma:uniform-convergence-cdl}.
\end{proof}

\subsection{Omniprediction Through Uniform-Mass Binning and Recalibration}\label[section]{section:omniprediction-from-bucketing}

We next analyze recalibration through uniform-mass binning. Binning is a long-established technique for measuring calibration \cite{miller1962statistical,sanders1963subjective}. The method of uniform-mass binning was introduced by \cite{zadrozny2001obtaining} as the first binning-based approach not only for measuring calibration, but also for obtaining a calibrated predictor. We show that this natural algorithm yields omniprediction with respect to the class of generalized monotone post-processings. In other words, recalibration via uniform-mass binning improves the performance of the input predictor under every proper loss, with improvement at least as large as that achieved by the best generalized monotone post-processing. While previous work focused on the calibration properties of uniform-mass binning, we show that this method preserves the information encoded in the predictor---when measured by proper losses---at least as well as any generalized monotone post-processing.

We note that the omniprediction guarantee achieved by this method is stronger than the one we established for the PAV algorithm in \Cref{theorem:omniprediction-through-pav}, since it applies to the larger class $\M_r$.

We define uniform-mass binning as follows.

\begin{definition}[Uniform-Mass Binning]\label[definition]{definition:quantile-bucketing}
    Let $P = (p_1,p_2,\dots,p_m)$ be a collection of $m$ points in $[0,1]$, with potential repetitions, and let $\eps\in (0,1)$. We say that a partition $(I_j)_{j\in [t]}$ of $[0,1]$ is an $\eps$-uniform-mass partition with respect to $P$ if $t \le 2/\eps$ and each $I_j$ in the partition is an interval (open, closed, or half-open) for which at least one of the following holds:
    \begin{itemize}
        \item (Small buckets) $|\{i\in [m]: p_i\in I_j\}| \le \eps m$.
        \item (Overflow buckets) The interval $I_j = [a,a]$ for some $a\in[0,1]$.
    \end{itemize}
\end{definition}
Overflow buckets are necessary since it could be the case that $p_i =a$ for say $1/2$ the samples, in which case we create a separate bucket for it. 
We show that recalibration with uniform-mass binning achieves omniprediction with respect to the class of $r$-generalized monotone post-processings, with a number of buckets that is linear in $r$. While the resulting predictor $\hat\kappa$ need not lie in $\M_r$, it is piecewise constant on $r' = O(r/\eps)$ intervals, hence it belongs to $\M_{r'}$.

\begin{theorem}\label[theorem]{theorem:omniprediction-through-recal}
    Let $\eps\in(0,1)$ and $r\ge 1$. Then, \Cref{algorithm:recal}, run with parameter $\eps' = \eps / 8r$ on  $O(r^2\log(1/\delta)/{\eps^4})$ samples learns an $(\eps,\M_r)$-omnipredictor with probability $1-\delta$ and has time complexity $O(r^2\log(1/\delta)/{\eps^4})$.
\end{theorem}

\begin{algorithm}

\caption{$\textsc{UMB-Recalibration}(S,\eps')$}\label{algorithm:recal}
\KwIn{Set $S$ of $m$ pairs of the form $(p,y)$ where $p\in[0,1], y\in\{0,1\}$, parameter $\eps'\in(0,1)$}
\KwOut{Post-processing function $\hat\kappa:[0,1]\to [0,1]$}
\BlankLine
Create an $\eps'$-uniform-mass partition $(I_j)_{j\in[t]}$ w.r.t. $(p_1,\dots,p_m)$ greedily, where $t \le 2/\eps'$\; 
For all $j\in[t]$, let 
\[ \hat\kappa_j = \frac{\sum_{i\in[m]} y_i \bm{1}[p_i\in I_j]}{\sum_{i'\in[m]} \bm{1}[p_{i'}\in I_j]}.\]\\
Return the function $\hat\kappa$ defined  as 
\[ \hat\kappa(p) = \sum_{j=1}^t \hat\kappa_j \bm{1}[p\in I_j].\]

\end{algorithm}

In order to prove \Cref{theorem:omniprediction-through-recal}, we will use the following result regarding the sampling errors. The first inequality here is a standard uniform convergence bound. The second one captures the intuition that the per-bucket average is likely to be close to accurate for all buckets that are reasonably large. The somewhat unusual formalization of this condition that we use below better fits our application later, and allows for a tighter bound on the sample complexity.

\newcommand{\acc}{\gamma}
\begin{lemma}\label[lemma]{lemma:convergence-bucketed-recalibration}
    Let $\acc\in (0,1)$, let $S$ be a set of $m$ i.i.d. samples from some distribution $J$ over $[0,1]\times \{0,1\}$, and let $(I_j, \hat\kappa_j)_{j\in [t]}$ be as defined in \Cref{algorithm:recal}. If $m \ge C\log(1/\delta)/{{\acc}^2}$ for some sufficiently large constant $C\ge 1$, then the following hold with probability at least $1-\delta$:
    \[
        \Bigr|\pr_J[p\in I_j] - \frac{1}{m}\sum_{i\in[m]}\bm{1}[p_i\in I_j] \Bigr| \le \acc, \text{ for all }j\in[t].
    \]
    \[
        |(\hat\kappa_j - \E_J[y|p\in I_j]) \cdot\Pr_J[p \in I_j]| \le \acc, \text{ for any }j\in[t];
    \]
\end{lemma}

\begin{proof}
    The intervals $(I_j)_{j\in[t]}$ in the output of the algorithm depend on the input examples $(p_i,y_i)_{i\in[m]}$ which are drawn i.i.d. from some distribution $J$. Using standard uniform convergence arguments (as for proving \Cref{lemma:uniform-convergence-cdl}), and the fact that the VC dimension of intervals over $[0,1]$ is $2$, we have that with probability at least $1-\delta$, the following guarantees hold as long as $m \ge {C} \log(1/\delta) / {{\acc}^2}$ for some sufficiently large constant $C\ge 1$.
    \begin{equation}
        \Bigr|  \frac{1}{m}\sum_{i\in[m]}y_i \bm{1}[p_i\in I_j] - \E_{(p,y)\sim J}[y \bm{1}[p\in I_j]]  \Bigr| \le \frac{{\acc}}{2}\,, \text{ for all }j\in[t]\label{equation:unif-conv-bucketing}
    \end{equation}
    \begin{equation}
        \Bigr|  \frac{1}{m}\sum_{i\in[m]}\bm{1}[p_i\in I_j] - \pr_{(p,y)\sim J}[p\in I_j]  \Bigr| \le \frac{{\acc}}{2} \,, \text{ for all }j\in[t]. \label{equation:unif-conv-bucketing-probabilities}
    \end{equation}

Fixing a $j \in [t]$, we write:
\begin{align*}
    |(\hat\kappa_j - \E_J[y|p\in I_j]) \cdot\Pr_J[p \in I_j]| &= |\hat\kappa_j\cdot\Pr_J[p \in I_j] - \E_J[y\cdot \bm{1}[p\in I_j]]|\\
    &= \left|\frac{\sum_{i\in[m]} y_i \bm{1}[p_i\in I_j]}{\sum_{i\in[m]} \bm{1}[p_{i}\in I_j]} \cdot \Pr_J[p \in I_j] - \E_J[y\cdot \bm{1}[p\in I_j]]\right|\\
    &\leq \left|\frac{\sum_{i\in[m]} y_i \bm{1}[p_i\in I_j]}{\sum_{i\in[m]} \bm{1}[p_{i}\in I_j]} \cdot \Pr_J[p \in I_j] - \frac{1}{m}\sum_{i\in[m]} y_i \bm{1}[p_i\in I_j]\right| \\
    &\;\;\;\; + \left|\frac{1}{m}\sum_{i\in[m]} y_i \bm{1}[p_i\in I_j] - \E_J[y\cdot \bm{1}[p\in I_j]]\right|\\
    &= \left|\frac{\sum_{i\in[m]} y_i \bm{1}[p_i\in I_j]}{\sum_{i\in[m]} \bm{1}[p_{i}\in I_j]}\right|\cdot \left|\Pr_J[p \in I_j] - \tfrac{1}{m}\sum_{i\in[m]} \bm{1}[p_{i}\in I_j]\right| \\
    &\;\;\;\; + \left|\frac{1}{m}\sum_{i\in[m]} y_i \bm{1}[p_i\in I_j] - \E_J[y\cdot \bm{1}[p\in I_j]]\right|\\
\end{align*}
    Here in the first step, we have used the definition of $\hat\kappa_j$, and used triangle inequality in the second step. The second term in the final expression is easily seen to be at most $\acc/2$ by \cref{equation:unif-conv-bucketing}. Finally note that \[ \left|\frac{\sum_{i\in[m]} y_i \bm{1}[p_i\in I_j]}{\sum_{i\in[m]} \bm{1}[p_{i}\in I_j]}\right|\leq 1\] 
    for every instantiation of the samples, so that \cref{equation:unif-conv-bucketing-probabilities} implies a bound of  $\acc/2$ on the first term as well.    
    This concludes the proof of \Cref{lemma:convergence-bucketed-recalibration}.
\end{proof}

We are now ready to prove \Cref{theorem:omniprediction-through-recal}. At a high level, we will split the decision loss for a loss $\ell_v$ and a post-processing function $\kappa$ across the buckets. We argue that except for a small number ($O(r)$) of buckets, the sign of $(\kappa(p)-v)$ is constant within the bucket, using the fact that $\kappa \in \M_r$. We separately handle these buckets using the fact that each such bucket (which cannot be an overflow bucket) has small probability mass. For the typical buckets that don't have a disagreement, the recalibration, if population-exact, would immediately ensure no decision loss.  We will use \cref{lemma:convergence-bucketed-recalibration} to control the error that arises due to the potential inaccuracy of the sampling-based recalibration. 

\begin{proof}[Proof of \Cref{theorem:omniprediction-through-recal}]
    With foresight, we set 
    \[ \eps' = \frac{\eps}{8r}, \ \acc = \frac{\eps\eps'}{4} = \frac{\eps^2}{32r}.\] 
    
    Our goal is to show the following inequality for all $\kappa\in\M_r$ and $\ell\in\L^*$:
    \begin{equation}
        \E_{(p,y)\sim J}[\ell(\hat\kappa(p), y)] \le \E_{(p,y)\sim J}[\ell(\kappa(p), y)] + \eps \label{equation:bucketing-goal}
    \end{equation}
    By \Cref{theorem:uv-cal}, any proper loss $\ell\in \L^*$ can be written as follows: 
    \[\ell(p,y) = \int_{[0,1]} \ell_v^+(p,y) \,d\mu_\ell^+(v) + \int_{[0,1]} \ell_v^-(p,y) \,d\mu_\ell^-(v)\, ,\] where $\mu_\ell$ is some measure over $[0,1]$ with $\mu_\ell^+([0,1]) +\mu_\ell^-([0,1]) \le 1$ and $\ell_v^{\pm}(p,y) = -(y-v)\sign_\pm(p-v)$. Therefore, it suffices to show \eqref{equation:bucketing-goal} for the losses $(\ell_v^{\pm})_{v\in [0,1]}$. Below, we will focus on the losses $(\ell_v^+)_{v\in [0,1]}$, since the proof for $(\ell_v^-)_{v\in [0,1]}$ will be analogous.

    \paragraph{Controlling Buckets of Disagreement.} We fix an arbitrary $\kappa\in \M_r$ and $v\in[0,1]$, and split $[t]$ into two parts as follows. We let $H= H_{\kappa,v} \subseteq[t]$ be as follows.
    \begin{equation}
        H = \Bigr| \Bigr\{ j\in[t] : \exists p,q\in I_j \text{ such that }\sign_+(\kappa(p)-v) \neq \sign_+(\kappa(q)-v)\Bigr\} \Bigr|\label{equation:disagreement-points}
    \end{equation}
    A key observation is that $|H| \le 2r$, due to the definition of $\M_r$ (\Cref{definition:generalized-monotonicity}). In particular, since $\{p: \kappa(p) \ge v\}$ can be expressed as a union of at most $r$ disjoint intervals $\I$, the number of sign changes of $\kappa(p)-v$ as $p$ increases from $0$ to $1$ is at most $2r$ (at the endpoints of the intervals). Moreover, any $j\in H$ corresponds to at least one distinct point of sign change for $\kappa(p)-v$, and, hence, $|H| \le 2r$. 
    
    Note that whenever $|I_j| = 1$, we have $j\not\in H$. Recall that, due to the construction of $(I_j)_{j\in [t]}$ (\Cref{algorithm:recal}), for any $j$ such that $|I_j| > 1$ we have $|I_j \cap \{p_i: i\in [m]\}|\le \eps' m$. Due to the uniform convergence bound of \Cref{lemma:convergence-bucketed-recalibration} we overall have:
    \begin{equation}
        \pr_{J}[p\in I_j] \le \eps' + \acc \text{ for all }j\in H\label{equation:bounding-probabilities-in-S}
    \end{equation}

    \paragraph{Handling typical buckets.} For an interval $I_j$ not in $H$, the $\sign_+(\kappa(p)-v)$ is constant throughout the interval, and the same is true for $\hat\kappa$, by construction. Let $s_j = \sign_+(\kappa(p) - v)$ for $p\in I_j$, $j\in [t]\setminus H$ and similarly $\hat s_j = \sign_+(\hat\kappa(p) - v)$. For such a bucket, we can write
    \begin{align*}
         & \E[\ell_v^+(\hat\kappa(p),y) \cdot \bm{1}[p\in I_j]] - \E[\ell_v^+(\kappa(p),y) \cdot \bm{1}[p\in I_j]] \\
         & \ \ = \E[-(\E[y|\bm{1}[p\in I_j] - v)\cdot (\hat s_j - s_j) \cdot \bm{1}[p \in I_j]]\\
         & \ \ = \E[-(\hat\kappa_j - v)\cdot (\hat s_j - s_j) \cdot \bm{1}[p \in I_j]] + \E[-(\E[y|\bm{1}[p\in I_j] - \hat\kappa_j)\cdot (\hat s_j - s_j) \cdot \bm{1}[p \in I_j]].
    \end{align*}
    Observe that  $-(\hat\kappa_j - v )\hat s_j = -|\hat\kappa_j-v| \le -(\hat\kappa_j - v ) s_j$ so that the first term is bounded above by zero.
    On the other hand, the second term is controlled by \Cref{lemma:convergence-bucketed-recalibration}:
    \begin{align*}
        \E[-(\E[y|\bm{1}[p\in I_j] - \hat\kappa_j)\cdot (\hat s_j - s_j) \cdot \bm{1}[p \in I_j]] & \leq 2\cdot \E[(\E[y|\bm{1}[p\in I_j] - \hat\kappa_j)\cdot \bm{1}[p \in I_j]]\\
        &\leq 2\acc.
    \end{align*}

    Thus we have shown that for any interval $I_j \not\in H$:
    \begin{equation}
    \E[\ell_v^+(\hat\kappa(p),y) \cdot \bm{1}[p\in I_j]] - \E[\ell_v^+(\kappa(p),y) \cdot \bm{1}[p\in I_j]] \leq 2\acc \label{equation:typical-buckets}
    \end{equation}

    \paragraph{Putting it Together.} We are now ready to prove the theorem. We will split the buckets into the buckets of disagreement and the rest, and use the bounds above to control the total error. We write
    \begin{align*}
        \E_{(p,y)\sim J}[&\ell(\hat\kappa(p), y)] - \E_{(p,y)\sim J}[\ell(\kappa(p), y)]\\
        &= \sum_{j \in [t]}\E_{(p,y)\sim J}[\left(\ell(\hat\kappa(p), y) - \ell(\kappa(p), y)\right)\cdot \bm{1}[p \in I_j]]\\
         &= \sum_{j \in H}\E_{(p,y)\sim J}[\left(\ell(\hat\kappa(p), y) - \ell(\kappa(p), y)\right)\cdot \bm{1}[p \in I_j]] \\
         &\;\;\;\; + \sum_{j \in [t] \setminus H}\E_{(p,y)\sim J}[\left(\ell(\hat\kappa(p), y) - \ell(\kappa(p), y)\right)\cdot \bm{1}[p \in I_j]]\\
         &\leq 2r \cdot (\eps' + \acc) + (2/\eps')\cdot (2\acc)\\
         &\leq 4r\eps' + 4\acc/\eps'\\
         &\leq \eps.
    \end{align*}
    Here we have bounded the sum over $H$ and $[t]\setminus H$ using \cref{equation:bounding-probabilities-in-S} and \cref{equation:typical-buckets} respectively. The claim follows.
 \end{proof}

\newpage
\bibliographystyle{alpha}
\bibliography{refs}

\newpage
\appendix

\section{Why Generalized Monotone Functions}
\label[section]{sec:why-gmr}

Here we discuss why generalized monotone functions are a natural class of post-processing functions. 
Intuitively, monotonicity is a natural constraint on  post-processing functions to apply to predictors  if we believe that the predictors are reasonably good to begin with: if $p = 0.7$, that ought to mean that the probability of the label $1$ is higher than if $p = 0.3$. However, this might not be true for  two close-by values like $0.7$ and $0.6999$. Thus it is natural to relax the condition and allow some violations of monotonicity (see \Cref{fig:crossings} for an example).  

A first attempt might be through allowing small total variation. 
An  increasing sequence $I_n$ of length $n$ in $[0,1]$ is $(p_i)_{i \in [n]}$ such that $0 \leq p_1 < p_2 \cdots < p_n \leq 1$. 
For a function $\kappa:[0,1] \to [0,1]$, we define its total variation as
\[ \mathrm{tv}(\kappa) = \lim_{n \to \infty}\sup_{I_n}\sum_{i=1}^{n-1}|\kappa(p_{i+1}) - \kappa(p_i)|\]
where the supremum is over all increasing sequences $I_n$ of length $n$.
The class $\mathrm{tv}(f) \leq 1$ contains both monotone and Lipschitz functions. But we have seen that $\vc(\threshold(\Lip)) = \infty$ whereas $\vc(\threshold(\M_+)) =1$, so they are very different from $\cdl$ viewpoint. 

Is there a strengthening of $\mtv$ that rules out arbitrary Lipschitz function but extends monotone functions? We show that $\M_r$ is such a class. 

\begin{definition}
For $v \in (0,1)$, let the {\em crossing number} at $v$ denoted  $\mathrm{cr}_v(\kappa)$ denote the largest $n$ for which there is a strictly increasing sequence $(p_i)_{i \in [n+1]}$ such that $\sign(\kappa(p_i) -v)$ alternates. Let 
\begin{align*} 
\mcr(\kappa) &= \sup_{v \in [0,1]} \mathrm{cr}_v(\kappa),\\
\mcr(\K) &= \sup_{\kappa \in \K}\mcr(\kappa)
\end{align*}
\end{definition}

It is easy to show that $\mcr(f)$ enjoys the following properties:
\begin{itemize}
    \item $\mathrm{cr}(\kappa)$ is within a constant factor of the smallest $r$ for which $\kappa \in \M_r$. 
    \item We have $\vc(\threshold(\K)) \leq \mathrm{cr}(\K)$, since the alternating sequence of signs on $\mathrm{cr}(\K) + 1$ points is not realizable within $\threshold(\K)$. 
    \item $\mtv(\kappa) \leq \mcr(\kappa)$. We can see $\mtv(\kappa)$ as a bound on $\E_v[\mcr_v(\kappa)]$ over uniformly random $v \in [0,1]$, whereas  $\mcr(\kappa)$  bounds $\mcr_v(\kappa)$ in the worst case. For more detail, see the following lemma.
\end{itemize}

\begin{lemma}
    $\mtv(\kappa) \le \mcr(\kappa)$ for all $\kappa : [0, 1] \to [0, 1]$.
\end{lemma}
\begin{proof}
    Fix any increasing sequence $I_n$ of points $0 \le p_1 < p_2 < \cdots < p_n \le 1$. For any particular index $2 \le i \le n$, the probability over a uniformly random value $v \sim [0, 1]$ that $\sign(\kappa(p_{i-1}) - v) \neq \sign(\kappa(p_i) - v)$ is precisely $\abs{\kappa(p_i) - \kappa(p_{i-1})}$. Let $\mcr_v(\kappa, p)$ denote the number of consecutive indices for which this sign change occurs. Then, by linearity of expectation,
    \[
        \E_{v} [\mcr_v(\kappa, p)] = \sum_{i=2}^n \Pr_v\bigl[\sign(\kappa(p_{i-1})-v) \neq \sign(\kappa(p_i)-v)\bigr] = \sum_{i=2}^n \abs{\kappa(p_{i-1})-\kappa(p_i)}.
    \]
    The supremum of the right side over all sequences $I_n$ is precisely $\mtv(\kappa)$, by definition. The supremum of the left hand side over all sequences $I_n$ is
    \[
        \sup_{I_n} \E_v[\mcr_v(\kappa, p)] \le \sup_{v} \sup_{I_n} \mcr_v(\kappa, p) = \sup_v \mcr_v(\kappa) = \mcr(\kappa).
    \]
\end{proof}

Thus $\M_r =\{\kappa: \mcr(\kappa) \leq O(r)\}$ corresponds to a class of functions that relaxes monotonicity, but excludes functions whose thresholds have high VC dimension, like Lipschitz functions.

\section{Additional Proofs}

\subsection{Additional Proofs From \cref{sec:preliminaries}}
\label[appendix]{sec:prelim-proofs}

\begin{proof}[Proof of \cref{thm:proper-losses-partial-improvement}]
    Take $\varphi, \varphi'$ as in the characterization in \cref{thm:proper-characterization}, so that
    \[\ell(p,q) = \varphi(p) + \varphi'(p)(q - p).\]
    Then we have
    \begin{align*}
        \ell(a, c) - \ell(b, c) &= \varphi(a) + \varphi'(a)(c - a) - \Bigl(\varphi(b) + \varphi'(b)(c - b)\Bigr) \\
        &= \underbrace{\varphi(a) + \varphi'(a)(b-a) - \varphi(b)}_{(*)} + \underbrace{\Bigl(\varphi'(a) - \varphi'(b)\Bigr)}_{(**)}(c - b),
    \end{align*}
    where nonnegativity of $(*)$ and $(**)$ follow from concavity of $\varphi$. The other inequality is analogous.
\end{proof}

\begin{proof}[Proof of \cref{thm:prop-loss-range}]
    By \cref{thm:proper-losses-partial-improvement}, the function $\ell(p, 0)$ is nondecreasing in $p$, and the function $\ell(p, 1)$ is nonincreasing in $p$. We also have $\abs{\ell(p, 1) - \ell(p, 0)} \le 1$ for any $\ell \in \L^*$. Thus, for any $0 \le p \le q \le 1$,
    \[
        \ell(p, 0) \le \ell(q, 0) \le \ell(q, 1) + 1.
    \]
    Considering more cases like this, we see that $\abs{\ell(p, 0) - \ell(q, 1)} \le 1$ for all $p, q \in [0, 1]$. Combining this inequality with $\abs{\ell(p, 1) - \ell(p, 0)} \le 1$, we see that any two outputs of $\ell$ on distinct inputs differ by at most $2$. Equivalently, we conclude that the range of any $\ell \in \L^*$ is contained in an interval of length at most $2$.
\end{proof}

\begin{proof}[Proof of \cref{corollary:uv-cal}]
    First, by \cref{theorem:uv-cal}, we know that any loss $\ell \in \L^*$ that is not of the form $\ell_v^+$ or $\ell_v^-$ or a constant function is redundant and need not be considered in the supremum. Next, for any fixed $(p, y) \in [0, 1] \times \{0, 1\}$ and $v \in [0, 1)$, we have $\lim_{w \to v^+} \ell_w^+(p, y) = \ell_v^-(p, y)$. Therefore, by the bounded convergence theorem, the losses $\ell_v^-$ for $v \in [0, 1)$ are redundant if we take the supremum over all $\ell_w^+$ for $w \in [0, 1]$. Note that $\ell_1^-(p, y) = y - 1$ is also redundant because it does not depend on $p$. Similarly, if $v \in (0, 1]$, then $\lim_{w \to v^-} \ell_w^-(p, y) = \ell_v^+(p, y)$. Therefore, the losses $\ell_v^+$ for $v \in (0, 1]$ are redundant if we take the supremum over $\ell_w^-$ losses for $w \in [0, 1]$. Note that $\ell_0^+(p, y) = -y$ is also redundant because it does not depend on $p$. The same is true of constant functions.
\end{proof}

\subsection{Additional Proofs From \cref{sec:samp-comp}}
\label[appendix]{app:samp-comp}

\begin{proof}[Proof of \Cref{lemma:uniform-convergence-cdl}]
    We define the following quantities for $v\in [0,1]$ and $\kappa\in \K$.
    \begin{align*}
        Q_{v,\kappa}^+ &= \E_{(p,y)\sim J}[\ell_v^+(p,y) - \ell_v^+(\kappa(p),y)] \text{ and } Q_{v,\kappa}^- = \E_{(p,y)\sim J}[\ell_v^-(p,y) - \ell_v^-(\kappa(p),y)] \\
        \hat{Q}_{v,\kappa}^+ &= \E_{(p,y)\sim S}[\ell_v^+(p,y) - \ell_v^+(\kappa(p),y)] \text{ and }\hat{Q}_{v,\kappa}^- = \E_{(p,y)\sim S}[\ell_v^-(p,y) - \ell_v^-(\kappa(p),y)]
    \end{align*}
    Due to \Cref{corollary:uv-cal}, it suffices to show that, with probability at least $1-\delta$ over the choice of $S$, we have the following:
    \begin{equation}
        |Q_{v,\kappa}^{\star} - \hat{Q}_{v,\kappa}^{\star}| \le \eps \,, \text{ for all } v\in[0,1], \kappa\in\K , \star\in \{+,-\}\label{equation:unif-conv}
    \end{equation}
    For the remainder of the proof, we will focus on $|Q_{v,\kappa}^{+} - \hat{Q}_{v,\kappa}^{+}|$, since the other case follows identically, with the additional observation that, for translation-invariant classes $\K$, we have $\vc(\threshold(\K)) = \vc(\threshold_-(\K))$, where $\threshold_-(\K) = \{p\mapsto \sign_-(\kappa(p) - 1/2): \kappa\in \K\}$.

    We will show that condition \eqref{equation:unif-conv} holds with probability at least $1-\delta$ as long as the size of $S$ is $m \ge \frac{Cd}{\eps^2} \log(\frac{1}{\eps \delta})$ for some large enough constant $C\ge 1$. To this end, recall that $\ell_v^+(p,y) = -(y-v) \sign_+(p-v)$.  Therefore:
    \begin{align*}
        Q_{v,\kappa}^+ &= \E_{(p,y)\sim J}[(y-v) (\sign_+(\kappa(p) - v) - \sign_+(p - v))] \\
        \hat{Q}_{v,\kappa}^+ &= \E_{(p,y)\sim S}[(y-v) (\sign_+(\kappa(p) - v) - \sign_+(p - v))]
    \end{align*}
    For the following, we use $\sign$ to denote $\sign_+$ and $Q_{v,\kappa}, \hat Q_{v,\kappa}$ to denote $Q_{v,\kappa}^+, \hat Q_{v,\kappa}^+$.
    Since $\K$ is translation invariant, for any $\kappa\in \K$ and $v\in[0,1]$, there is $\kappa'\in\K$ such that $\kappa'(p) = [\kappa(p) + \frac{1}{2} - v]_0^1$. Then, we have $\sign(\kappa(p)-v) = \sign(\kappa(p)+ 1/2-v - 1/2) = \sign(\kappa'(p) - 1/2)$, for all $p\in[0,1]$, and, therefore:
    \begin{align*}
        Q_{v,\kappa} &= \E_{(p,y)\sim J}[(y-v) (\sign(\kappa'(p) - 1/2) - \sign(p - v))] \\
        \hat{Q}_{v,\kappa} &= \E_{(p,y)\sim S}[(y-v) (\sign(\kappa'(p) - 1/2) - \sign(p - v))]
    \end{align*}

    Let $V = \{\frac{i \epsilon}{16}: i = 0, 1, \dots, \lfloor\frac{16}{\eps}\rfloor \} \cup \{1\}$, and $\proj_{\eps}(v) = \arg \min_{v'\in V}|v-v'|$. We define the following quantities.
    \begin{align*}
        Q_{v,\kappa}^V &= \E_{(p,y)\sim J}[(y-\proj_{\eps}(v)) (\sign(\kappa'(p) - 1/2) - \sign(p - v))] \\
        \hat{Q}_{v,\kappa}^V &= \E_{(p,y)\sim S}[(y-\proj_{\eps}(v)) (\sign(\kappa'(p) - 1/2) - \sign(p - v))]
    \end{align*}
    Note that, by the choice of $V$, we have $|Q_{v,\kappa}^V - Q_{v,\kappa}| \le \eps/8$, and, similarly, $|\hat{Q}_{v,\kappa}^V - \hat{Q}_{v,\kappa}| \le \eps/8$, for all $v\in [0,1]$ and $\kappa\in \K$. Therefore, it suffices to show that with probability at least $1-\delta$, the following holds for any $v\in[0,1]$, any $v'\in V$, and any $\kappa' \in \K$:
    \[
        \Bigr| \E_{(p,y)\sim J}[(y-v') (\sign(\kappa'(p) - 1/2) - \sign(p - v))] - \E_{(p,y)\sim S}[(y-v') (\sign(\kappa'(p) - 1/2) - \sign(p - v))] \Bigr| \le \frac{\eps}{4}\,.
    \]
    Since the size of $V$ is $O(1/\eps)$, it suffices to prove that the above bound holds for each individual $v'\in V$ (but uniformly over $v\in[0,1]$ and $\kappa\in \K$) with probability $1-O(\eps \delta)$. The desired result would then follow by a union bound. For each $v'\in V$, the desired bound is true due to standard uniform convergence arguments, combined with the fact that the VC dimension of the class $\{p \mapsto \sign(\kappa(p)-1/2) : \kappa \in \K\}$ is $d$, and the VC dimension of the class $\{p \mapsto \sign(p-v) : v\in[0,1]\}$ is $1$. In particular, we may combine the following results from \cite{mohri2018foundations}: Corollary 3.8 (which gives a bound on the Rademacher complexity of a binary class in terms of the associated growth function), Theorem 3.17 (Sauer's Lemma, which bounds the growth function in terms of the VC dimension), and Theorem 11.3 (which gives a generalization bound for regression with respect to bounded and Lipschitz losses, in terms of the Rademacher complexity of the underlying function class). The choice of $O(\eps \delta)$ for the failure probability only incurs an additive term of $O(\log(1/\eps))$, and, therefore, our choice for the number of samples $m$ suffices to achieve the desired result.
\end{proof}

\begin{proof}[Proof of Theorem \ref{theorem:sample-complexity} Upper Bound]
    Let $\A_1$ be the algorithm that receives a set $S$ of $m$ i.i.d. examples from some unknown distribution $J$ over $[0,1]\times\{0,1\}$, where $m \ge C \cdot d\log(1/\eps)/\eps^2$, for some sufficiently large universal constant $C\ge 1$, and does the following:
    \begin{enumerate}
        \item Compute the following quantity:
        \[ \cdl_{\K}(S) = \sup_{\kappa \in \K, \ell\in\L^*} \frac{1}{m}\sum_{(p,y)\in S} (\ell(p,y) - \ell(\kappa(p),y)) \]
        \item If $\cdl_{\K}(S) \le \alpha - \epsilon/2$, then output $\accept$. Otherwise, output $\reject$.
    \end{enumerate}

    By \Cref{lemma:uniform-convergence-cdl}, with probability at least $2/3$, we have $|\cdl_{\K}(J) - \cdl_{\K}(S)| \le \eps/2$. Under this event, if $\cdl_{\K}(J) \le \alpha -\eps$, then algorithm $\A_1$ will accept. On the other hand, if $\cdl_{\K}(J) > \alpha$, then $\A_1$ will reject. 
\end{proof}

\begin{proof}[Proof of Corollary \ref{thm:test-audit-sc}]
    The proof of the upper bound is essentially the same as in \cref{theorem:sample-complexity}. The key ingredient of that proof was the uniform convergence bound stated in \cref{lemma:uniform-convergence-cdl}, which holds for a supremum over all $\ell \in \L^*$. In particular, it continues to hold if we only take the supremum over $\ell \in \L^*_{\mu\mathsf{-sc}}$, so the proof still goes through. For the lower bound, we can no longer rely on our argument based on the $\ell_{1/2}^+$ loss, since its associated function $\varphi_{1/2}(p) = -\abs{p-1/2}$ is not strongly concave. To circumvent this issue, consider the following version of the squared loss:
    \[
        \ell_{\mathsf{sq}}(p, y) = (y - p)^2.
    \]
    Then $\partial \ell_{\mathsf{sq}}(p) = 1 - 2p \in [-1, +1]$, so the function $\ell_{\mathsf{sq}}$ indeed belongs to the class $\L^*$. Moreover, $\varphi_{\mathsf{sq}}(p) = \ell_{\mathsf{sq}}(p, p) = p(1-p)$ is $2$-strongly concave, so $\ell_{\mathsf{sq}} \in \L^*_{2\mathsf{-sc}}$. Next, given any $\ell \in \L^*$ and $\mu > 0$, we define the following convex combination, which belongs to $\L^*_{\mu\mathsf{-sc}}$:
    \[
        \ell_\mu = \frac{\mu}{2}\ell_{\mathsf{sq}}  + \Bigl(1 - \frac{\mu}{2}\Bigr)\ell.
    \] We will use $\ell_\mu$ to study the effect of restricting to $\mu$-strongly convex proper losses. First, since we are considering a restricted class of loss functions, we clearly have $\cdl_{\L^*_{\mu\mathsf{-sc}},\K} \le \cdl_\K$. Conversely, let $\kappa$ be a post-processing function that improves $\ell$ by at least $\alpha$, meaning that
    \[
        \E\bigl[\ell_{\mathsf{sq}}(\kappa(p), y)\bigr] \le \E\bigl[\ell_{\mathsf{sq}}(p, y)\bigr] - \alpha.
    \]
    Since the convex combination $\ell_\mu$ puts $1 - \mu/2$ weight on $\ell$, the post-processing $\kappa$ must improve the loss on this part of $\ell_\mu$ by at least $(1-\mu/2)\alpha$. Although $\kappa$ may worsen the loss on $\ell_{\mathsf{sq}}$ arbitrarily, the convex combination $\ell_\mu$ puts only $\mu/2$ weight on $\ell_{\mathsf{sq}}$, so it must worsen the loss on this part of $\ell_\mu$ by at most $\mu_2 \cdot 2$ (recall that all losses in $\L^*$, such as $\ell_{\mathsf{sq}}$ have range bounded in an interval of length $2$, by \cref{thm:prop-loss-range}). In total, $\kappa$ must improve the loss on $\ell_\mu$ by at least
    \[
        \Bigl(1 - \frac{\mu}{2}\Bigr)\alpha - \mu \ge \alpha - 2\mu.
    \]
    
    It follows immediately that an $(\alpha, \beta)$-auditor for $\cdl_{\K}$ is implied by an $(\alpha-2\mu,\beta)$-auditor for $\cdl_{\L^*_{\mu\mathsf{-sc}},\K}$. Thus, setting $\alpha = 1/8$ and $\beta = 0$, our lower bound for auditing carries over to the case of $\L^*_{\mu\textsf{-sc}}$, as claimed.
\end{proof}

\begin{proof}[Proof of Theorem \ref{thm:cdl-smooth}]
    We will prove the upper and lower bounds separately.

    \paragraph{Upper Bound.} From the loss OI lemma (\Cref{thm:loss-oi}), we have that $\ell(p,y) - \ell(\kappa(p),y) \le (\partial\ell(\kappa(p)) - \partial\ell(p) ) (y-p)$. The function $w'(p) = \partial\ell(\kappa(p)) - \partial\ell(p)$ is $6$-Lipschitz, because $\kappa$ is $2$-Lipschitz and $\partial \ell$ is $2$-Lipschitz. Therefore, we have
    \[
        \smcdl(J) \le \sup_{w':\; 6\text{-Lipschitz}} \E\Bigr[w'(p) (y-p)\Bigr] \le 6\cdot\smce(J)
    \]

    \paragraph{Lower Bound.} Suppose that there is a $1$-Lipschitz function $w:[0,1]\to [-1,1]$ such that 
    \[
        \smce(J)= \E_{(p,y)\sim J}[(y-p)\cdot w(p)] = \alpha\,.
    \]
    Then, the post-processing function $\kappa(p) = [p+\alpha\, w(p)]_0^1$ is $2$-Lipschitz and satisfies:
    \[
        \E_{(p,y)\sim J}\Bigr[(y-\kappa(p))^2\Bigr] \le \E_{(p,y)\sim J}\Bigr[(y-p)^2\Bigr] - \alpha^2
    \]
    The squared loss $\ell_{\mathsf{sq}}(p,y) = (p-y)^2/2$ is $1$-Lipschitz and proper. Therefore, $\smcdl(J) \ge \alpha^2/2$.
\end{proof}

\section{Tightness of the Weight-Restricted Calibration Characterization}
\label[appendix]{sec:tightness-of-cdl-vs-ce}

In this section, we show that the quadratic gap in \cref{thm:cdl-vs-ce} is essentially tight, using the example of the class of monotonically nondecreasing post-processings $\K = \M_+$.

\begin{theorem}
\label[theorem]{thm:tightness-of-cdl-vs-ce}
    There exist distributions $J_1, J_2$ over pairs $(p, y) \in [0, 1] \times \{0, 1\}$ such that
    \[
        \cdl_{\M_+}(J_1) \gtrsim \propce(J_1)
    \]
    and
    \[
        \cdl_{\M_+}(J_2) \lesssim \propce(J_2)^2.
    \]
\end{theorem}

\begin{proof}
    Our examples are essentially the same ones used by \cite{HuWu24} to establish the tightness of their relationship between $\cdl_{\Kall}$ and $\ece$, which also has a quadratic gap.

    For $J_1$, suppose that $p = 1 - \eps$ and $y = 1$ deterministically. Then,
    \[
        \propce(J_1) \le \ece(J_1) = \eps,
    \]
    but if we set $v = 1-\eps/2$ and consider the monotonic post-processing $\kappa(p) = p + \eps$, then
    \[
        \cdl_{\M_+}(J_1) \ge \E[\ell_{v}^+(p, y)] - \E[\ell_v^+(p + \eps, y)] = \frac{\eps}{2} - \Bigl(-\frac{\eps}{2}\Bigr) = \eps.
    \]
    Therefore, $\cdl_{\M_+}(J_1) \gtrsim \propce(J_1)$.

    For $J_2$, consider a uniform $p \sim [0, 1 - \eps]$, and suppose that $y|p \sim \Be(p + \eps)$. Then,
    \[
        \propce(J_2) \ge \E[y - p] = \eps,
    \]
    but the characterization of $\cdl$ in \cref{corollary:uv-cal} implies 
    \begin{align*}
        \cdl_{\M_+}(J_2) &= \sup_{\substack{v \in [0, 1], \\ \kappa \in \M_+}} \, \E[\ell_v^+(p, y)] - \E[\ell_v^+(p, y)] \\
        &=\sup_{v \in [0, 1]} \, \E[\ell_v^+(p, y)] - \E[\ell_v^+(p + \eps, y)] \\
        &= \sup_{v \in [0, 1]} \, \E[(\sign_+(p + \eps - v) - \sign_+(p - v))(p+\eps-v)] \\
        &= \sup_{v \in [0, 1]} \, 2 \cdot \E[\bm{1}[v - \eps \le p < v](p+\eps-v)]\\
        &\le \sup_{v \in [0, 1]} \, 2 \cdot \E[\bm{1}[v - \eps \le p < v] (v + \eps - v)] \\
        &= 2 \cdot \frac{\eps}{1-\eps} \cdot \eps \\
        &\lesssim \eps^2.
    \end{align*}
    Thus, $\cdl_{\M_+}(J_2) \lesssim \propce(J_2)^2$.
\end{proof}

\section{Properties of Generalized Monotone Post-Processings}\label[appendix]{section:appendix-generalized-monotone}

We provide here some classical results related to generalized monotone post-processings, which, in particular, imply \Cref{theorem:test-audit-generalized-monotone}. We begin with the following result on their VC dimension.

\begin{proposition}\label[proposition]{proposition:properties-of-generalized-monotone}
    For any $r\in \N$, $\vc(\threshold(\M_r))  = 2r$.
\end{proposition}

\begin{proof}
    Consider the points $p_i = \frac{i}{2r}$ where $ i \in [2r]$. We will show the following:
    \[
        \{\vv \in \cube{2r}: \vv = (\sign_+(\kappa(p_i) - 1/2))_{i\in[2r]} \text{ for some }\kappa\in\M_r\} = \cube{2r}
    \]
    In particular, for any $\vv\in \cube{2r}$, we form the intervals $I_1, I_2, \dots, I_r$ as follows. Let $I_1$ be an interval of the form $[p_i, p_j]$, where $i\le j$ and $\vv(i) = \vv(i+1) = \cdots = \vv(j) = 1$ and $\vv(i') = -1$ for all $i' < i$ and $i' = j+1$. We then proceed recursively to form $I_2, \dots, I_r$ after removing the points $p_1, \dots, p_j, p_{j+1}$. Note that for each interval we form, we remove at least $2$ points. One for $p_j$ and one for $p_{j+1}$. We let $\kappa_{\vv}(p) = \bm{1}[p \in \cup_{i\in [r]} I_i]$, where we have $\kappa_{\vv} \in \M_r$ and $(\sign_+(\kappa(p_i) - 1/2))_{i\in[2r]} = \vv$.

    On the other hand, for any set of $2r+1$ distinct points on $[0,1]$, no function in $\M_r$ can generate the labeling $\vv'\in \cube{2r+1}$, where $\vv'(2i-1) = 1$ and $\vv'(2i) = -1$, for $i = 1, 2,\dots, r$, and $\vv'(2r+1) = 1$, because the set $\{p:\kappa(p) \ge 1/2\}$ would then have at least $r+1$ disjoint components. 
\end{proof}

Our testing result in \Cref{theorem:test-audit-generalized-monotone}, which pertains to the specific class $\M_r$, is achieved by instantiating our more general \Cref{theorem:testing-through-proper-agnostic-learning} with the following standard agnostic learner for unions of $r$ intervals, based on dynamic programming (see Section 4.2 in \cite{KearnsSS94}). Indeed, the tester of \Cref{theorem:testing-through-proper-agnostic-learning} makes $O(\log(1/\eps\delta)/\eps)$ non-adaptive calls to the agnostic learner, which has sample complexity $O((r+\log1/\delta)/\eps^2)$ and runtime $\tilde{O}(r^2+r\log1/\delta)/\eps^2)$. Therefore, a union bound over the failure probability of each call implies a tester with total sample complexity $\tilde{O}(r/\eps^2)$ and runtime $\tilde{O}(r^2/\eps^3)$. For completeness, we describe and analyze the proper agnostic learner here.

\begin{theorem}[\cite{KearnsSS94}]\label[theorem]{theorem:erm-generalized-monotone}
    Let $r\ge 1$ and $\C = \threshold(\M_r)$. For any $\eps,\delta\in (0,1)$, there is a proper agnostic $(\eps,\delta)$-learner for $\C$ with sample complexity $O(({r +\log 1/\delta})/{\eps^2})$  and runtime $\tilde{O}( (r^2 + r\log 1/\delta)/\eps^2)$.
\end{theorem}

\begin{algorithm}[ht]
\caption{$\textsc{EstimateMinimumRisk}(S, r)$}\label{algorithm:erm-generalized-monotone}
\KwIn{Set $S$ of $m$ pairs of the form $(p,z)$ where $p\in[0,1], z\in\cube{}$ and $r\ge 1$.}
\KwOut{A value $R \in [0,1]$.}
\BlankLine
\tcc{Sorting the Samples, in time: $O(m \log(m))$}
\BlankLine
Let ${\cal O} = ((p_1,z_1),(p_2,z_2),\dots,(p_m,z_m))$ be an ordering such that $p_1\le p_2\le \dots\le p_m$\;
If there are any duplicates $p_i = p_{i+1}=\dots=p_{i+k}$, then merge them and set $z' = \sum_{j=0}^{k}z_{i+j}$, so that we obtain ${\cal O}' = ((p'_1,z'_1),(p'_2,z'_2),\dots,(p'_{m'},z'_{m'}))$ with $p'_1 < p'_2<\dots< p'_{m'}$\;
\BlankLine
\tcc{Initializations, in time: $O(m)$}
\BlankLine
Set $Z(i) = \sum_{j\le i} z_j'$, $Z(0) = 0$\;
Set $Q(0,t) = Q(s,1) = 0$ for all $t\in[m'+1]$ and $s \in [r]$\;
Set $M(0,j) = \max_{1\le i \le j}\{- Z(i-1)\}$ for all $j\in[m'+1]$ and $M(s,1) = 0$ for all $s\in[r]$\;
Set $B(0,t) = \max_{1\le j < t}\{Z(j) + M(0,j)\}$ for all $t\in [m'+1]$ and $B(s,1) = 0$ for all $s\in [r]$\;
\BlankLine
\tcc{Dynamic Programming, in time: $O(mr)$}
\BlankLine
\For{$s = 1,2,\dots, r$}{
    \For{$t = 2,3,\dots, m'+1$}{
        $Q(s,t) = \max\{Q(s-1,t), B(s-1,t)\}$\;
    }
    \For{$j = 2,3,\dots, m'+1$}{
        $M(s,j) = \max\{ M(s, j-1), Q(s-1,j) - Z(j-1)\}$\;
        $B(s,j) = \max\{ B(s,j-1), Z(j) - M(s,j)\}$\;
    }
}
Let $R = Q(r,m'+1) / m$\;
\end{algorithm}

\begin{proof}[Proof of \Cref{theorem:erm-generalized-monotone}]
    We will show that there is an algorithm $\erm(\C)$ that takes as input a set $S$ of labeled examples of the form $(p,z)$ where $p\in [0,1]$ and $z\in\{\pm 1\}$, runs in time $O(|S| r + |S|\log |S|)$, and outputs some $h\in \C$ such that the following holds:
    \[
        \pr_{(p,z)\sim S}[h(p) \neq z] \le \min_{f\in \C} \pr_{(p,z)\sim S}[f(p) \neq z]
    \]
    In \Cref{algorithm:erm-generalized-monotone} we present a version of $\erm(\C)$ that does not return $h$, but only returns an estimate of its error. The algorithm $\erm(\C)$ can be implemented based on \Cref{algorithm:erm-generalized-monotone} by using some additional space.
    The agnostic learning result then follows by standard uniform convergence arguments, due to the fact that $\vc(\C) = 2r$ (\Cref{proposition:properties-of-generalized-monotone}), as long as $m:=|S| \ge C(r+\log(1/\delta))/\eps^2$ for some sufficiently large constant $C\ge 1$.
    
    We first observe that $\bm{1}[z \neq f(p)] = ({1 - z\cdot f(p)})/{2}$ for any $f \in \C, z \in \cube{}$. Moreover, any $f\in \C$ can be expressed as follows for $r$ disjoint intervals $I_1,I_2,\dots,I_r\subseteq[0,1]$:
    \[
        f(p) = 2\sum_{i=1}^r \bm{1}[p\in I_i] - 1
    \]
    We therefore have the following:
    \begin{align*}
        \min_{f\in \C} \pr_{(p,z)\sim S}[f(p) \neq z] &= \frac{1}{2} - \frac{1}{2}\max_{f\in \C} \E_{(p,z)\sim S}[z f(p)] \\
        &= \frac{1}{2} + \frac{1}{2} \E_{(p,z)\sim S}[z] - \frac{1}{2} \max_{I_1, \dots, I_r} \sum_{i = 1}^r \E_{(p,z)\sim S}[z \bm{1}[p\in I_i])]
    \end{align*}

    It suffices to find the endpoints of the intervals $(I_i)_i$ that maximize the following quantity:
    \[
        Q := \max_{I_1, \dots, I_r} \sum_{i = 1}^r \sum_{(p,z)\in S}z \bm{1}[p\in I_i]
    \]
    Let $S'$ be the set of pairs of the form $(p,z')$, where each $p\in[0,1]$ appears in $S$ at least once and $z'$ is the sum of all $z$ such that $(p,z)\in S$. Each $p\in [0,1]$ appears in $S'$ at most once. Let $P$ be the set of $p\in[0,1]$ appearing in $S'$. Then, we can write $Q$ as follows:
    \[
        Q = \max_{k\in [r]}\max_{\substack{a_i,b_i\in P \\ a_i \le b_i < a_{i+1}}} \sum_{i = 1}^k \sum_{(p,z')\in S'} z' \bm{1}[p\in [a_i,b_i]]\,.
    \]
    Let $P(q) = P\cap [0,q)$ and let $Q(k,q)$ be defined as follows:
    \[
        Q(k,q) := \max_{k' \le k} \max_{\substack{a_i,b_i\in P(q) \\ a_i \le b_i < a_{i+1}}} \sum_{i = 1}^{k'} \sum_{(p,z')\in S'}z' \bm{1}[p\in [a_i,b_i]]\,.
    \] 
    Then, $Q(k,q)$ satisfies the following recurrence for all $k \in [r]$ and $q\in P\cup \{\infty\}$:
    \[
        Q(k,q) = \max\Biggr\{ Q(k-1,q) , \max_{\substack{a,b\in P(q) \\ a \le b }} \Biggr\{ Q(k-1, a) + \sum_{(p,z')\in S'} z' \bm{1}[p\in [a,b]] \Biggr\}  \Biggr\}\,,
    \]
    as long as the initial conditions are the following for all $q\in P\cup \{\infty\}$ and $k\in [r]$:
    \begin{align*}
        Q(0, q) = Q(k, p_{\min}) = 0\,, \text{ where }p_{\min} = \min_{p\in P}p
    \end{align*}
    Our goal is to estimate the quantity $Q = Q(r,\infty)$, and retrieve the points $(a_i,b_i)_i$ that achieve the maximum. Due to the structure of $Q(k,q)$, the estimation of $Q$ can be achieved in time $\poly(r, |P|)$ via dynamic programming. For the indices $(a_i,b_i)_i$, we associate each pair $(k,q)$ with a set of intervals $\I(k,q)$ of the form $[a,b]$, so that:
    \[
        \I(k,q) = 
        \begin{cases}
            \I(k-1,q)\,, \text{ if }Q(k,q) = Q(k-1, q) \\
            \I(k-1,a)\cup \{[a,b]\}\,, \text{ if }Q(k,q) = Q(k-1, a) + \sum_{(p,z')\in S'} z' \bm{1}[p\in [a,b]]
        \end{cases}
    \]
    In order to obtain an improved time complexity, we may use some additional memory to store intermediate auxiliary quantities as described in \Cref{algorithm:erm-generalized-monotone}. In particular, we define the following quantities for all $a\in P$:
    \[
        Z(a) = \sum_{\substack{(p,z')\in S'}} z'\bm{1}[p\le a], \text{ and } \prev(a) = \max_{a'\in P(a)} a'
    \]
    and we equivalently formulate $Q(k,q)$ as follows:
    \begin{align*}
        Q(k,q) &= \max\Biggr\{ Q(k-1,q) , \max_{\substack{a,b\in P(q) \\ a \le b }} \Biggr\{ Q(k-1, a) + Z(b) - Z(\prev(a)) \Biggr\}  \Biggr\} \\
        &= \max\Biggr\{ Q(k-1,q) , \max_{\substack{b\in P(q)}} \Biggr\{ Z(b) + \max_{\substack{ a \le b }}\Biggr\{ Q(k-1, a) - Z(\prev(a)) \Biggr\}  \Biggr\} \Biggr\}
    \end{align*}
    where we may define the following quantities:
    \begin{align*}
        M(k-1, b) &= \max_{\substack{ a \le b }}\Biggr\{ Q(k-1, a) - Z(\prev(a)) \Biggr\} \\
        B(k-1, q) &= \max_{\substack{b\in P(q)}} \Biggr\{ Z(b) + M(k-1, b)  \Biggr\}
    \end{align*}
    Due to the definitions of the auxiliary quantities, they satisfy the following recurrence relations.
    \begin{align*}
        Q(k,q) &= \max\{ Q(k-1,q) , B(k-1, q)\} \\
        B(k,q) &= \max\{ B(k,\prev(q)), Z(q)+ M(k,q)\} \\
        M(k,q) &= \max\{ M(k,\prev(q)), Q(k-1, q) - Z(\prev(q)) \}
    \end{align*}
    Each step of the above recurrence relations can be computed in $O(1)$, which implies the desired bound on the runtime.
\end{proof}

\end{document}